\RequirePackage{fix-cm}
\documentclass[smallextended]{svjour3}       
\smartqed  
\usepackage{amsmath}
\usepackage{graphicx,psfrag,epsf}
\usepackage{enumerate}

\usepackage{amssymb}
\usepackage{mathtools}

\usepackage{enumitem}
\usepackage{enumerate}
\usepackage{amssymb}
\usepackage{amsmath, amsthm}
\usepackage{caption}
\usepackage{subcaption}
\usepackage[backend=bibtex,style=alphabetic, citestyle=authoryear]{biblatex}
\usepackage{stackengine, graphicx, xcolor}
\usepackage{tikz}
\usepackage{pgfplots}
\usepackage{helvet,url, hyperref}
\newtheorem{assumption}{Assumption}

\theoremstyle{definition}
\theoremstyle{example}
\theoremstyle{remark}

\DeclarePairedDelimiter\ceil{\lceil}{\rceil}
\DeclarePairedDelimiter\floor{\lfloor}{\rfloor}

\begin{document}

\title{Minimum discrepancy principle strategy for choosing $k$ in $k$-NN regression
}


\author{Yaroslav Averyanov         \and
        Alain Celisse 
}


\institute{Corresponding author:\\ Y. Averyanov \at
              Inria MODAL project-team, Lille,
France \\
              \email{yaroslavmipt@gmail.com}           
           \and
           A. Celisse \at
              Paris 1 - Panthéon Sorbonne University, Paris, France;\\
              Inria MODAL project-team, Lille, France
}

\date{Received: date / Accepted: date}

\maketitle

\begin{abstract}
We present a novel data-driven strategy to choose the hyperparameter $k$ in the $k$-NN regression estimator without using any hold-out data. We treat the problem of choosing the hyperparameter as an iterative procedure (over $k$) and propose using an easily implemented in practice strategy based on the idea of \textit{early stopping} and the \textit{minimum discrepancy principle}. This model selection strategy is proven to be minimax-optimal over some smoothness function classes, for instance, the Lipschitz functions class on a bounded domain. The novel method often improves statistical performance on artificial and real-world data sets in comparison to other model selection strategies, such as the Hold-out method, $5$--fold cross-validation, and AIC criterion. The novelty of the strategy comes from reducing the computational time of the model selection procedure while preserving the statistical (minimax) optimality of the resulting estimator. More precisely, given a sample of size $n$, if one should choose $k$ among $\left\{ 1, \ldots, n \right\}$, and $\left\{ f^1, \ldots, f^n \right\}$ are the estimators of the regression function, the minimum discrepancy principle requires the calculation of a fraction of the estimators, while this is not the case for the generalized cross-validation, Akaike's AIC criteria, or Lepskii principle. 

\end{abstract}

\section{Introduction}
\label{sec:intro}
%
The theoretical performance of the $k$-NN regression estimator has been widely studied since the 1970s (\cite{devroye1978uniform, collomb1979estimation, devroye1981almost, bhattacharya1987weak, biau2010rates, kpotufe2011k, biau2015lectures, zhao2021minimax}). For example, in (\cite[Chapter 12]{biau2015lectures}) the uniform consistency of the $k$-NN estimator is proved under the condition that $k(n) / n \to 0$ as $n \to \infty$, where $n$ is the sample size. However, as it was shown in \cite{gyorfi2006distribution}, the nearest neighbor estimator ($k = 1$) is proved to be consistent only in the noiseless case. 

Recently, researchers started to be interested in choosing $k$ from the data (\cite{gyorfi2006distribution, arlot2009data, kpotufe2011k, azadkia2019optimal}). Apparently, the most common and simplest strategy to choose $k$ is to assume some smoothness assumption on the regression function (e.g., the Lipschitz condition \cite{gyorfi2006distribution}) and to find $k$ that makes an upper bound on the bias and the variance of the $k$-NN regression estimator equal. This method has a clear lack: one needs to know the smoothness of the regression function (e.g., the Lipschitz constant). \cite{2009arXiv0909.1884A} gave a \textit{data-driven strategy} for choosing a hyperparameter for different linear estimators (e.g., the $k$-NN estimator) based on the idea of minimal penalty, introduced previously in \cite{birge2007minimal}. The main inconvenience of this strategy is that one needs to compute all the estimators $\mathbf{F}_n = \{ f^k, \ k = 1, \ldots, n \}$ of the regression function in order to choose the optimal one among them by comparing them via a special criterion that involves the empirical error (least-squares loss). To list other (similar) strategies, one can think about the Akaike's AIC (\cite{akaike1998information}), Mallows's $C_p$ (\cite{mallows2000some}) criteria or generalized cross-validation (\cite{li1987asymptotic, hastie2009elements}), where one has to compute the empirical risk error plus a penalty term for any $k = 1, \ldots, n$, or Lepskii principle (\cite{lepski1992problems, lepskii1992asymptotically, birge2001alternative}), where the statistician should make pairwise comparisons between all estimators $\mathbf{F}_n$. Often it is computationally expensive and restricts the use in practice. This gives rise to the problem of choosing the hyperparameter "in real-time", meaning that the practitioner should compute iteratively $f^k \in \mathbf{F}_n$. Eventually, this iterative process has to be stopped. This problem can be solved by applying the \textit{early stopping rule} (\cite{engl1996regularization, zhang2005boosting, Yao2007, raskutti2014early, blanchard2018optimal, wei2017early}).

The first early stopping rule that could be potentially data-driven was proposed by (\cite{blanchard2018optimal, blanchard2018early, celisse2021analyzing}) for spectral filter iterative algorithms (see, e.g., \cite{bauer2007regularization, gerfo2008spectral} for examples of such algorithms). The idea behind the construction of this early stopping rule is the so-called \textit{minimum discrepancy principle} (MDP) that is based on finding a first iteration for which a learning algorithm starts to fit the noise. The key quantity for the analysis of the minimum discrepancy principle is the \textit{empirical risk error}, which is monitored throughout the whole learning process. The process is stopped if the empirical risk starts to fit the noise. It is important to emphasize that the MDP strategy does not involve any hold-out data, thus useful in the settings where there is little data to do training and prediction afterwards.


\textbf{Contribution.} In the present paper, we propose applying the minimum discrepancy principle stopping rule for the $k$-NN regression estimator in order to select $k$. We prove a non-asymptotic upper bound on the performance of the minimum dicrepancy principle stopping rule measured in the empirical $L_2(\mathbb{P}_n)$ and population $L_2(\mathbb{P}_X)$ norms. Those bounds imply that, under a quite mild assumption on the regression function, the minimum discrepancy principle stopping rule provides a statistically (minimax) optimal functional estimator, in particular, over the class of Lipschitz functions on a bounded domain. In Section 5, we show that our method often improves the performances of classical selection procedures, such as 5-fold, Hold-out and generalized cross-validation, on artificial and real data. Besides that, the proposed strategy lowers the computational time of the selection procedure compared to some well-known model selection strategies, such as the generalized cross-validation, Akaike's AIC or Lepskii principle, since the latter strategies need to calculate all $k$-NN estimators, $k = 1, \dots, n$, of the regression function. We emphasize that the proposed strategy for choosing $k$ is \textit{easy-to-implement} in practice since it involves only monitoring the empirical risk, does not use any hold-out data and does not require any parameters to tune, which is not true for (\cite{arlot2009data}), Lepskii principle (\cite{lepski1992problems, lepskii1992asymptotically}) or (\cite{zhao2019minimax, zhao2021minimax}). In the latter work, in order to calculate the optimal $k $, three design parameters $(A, K, q)$ have to be carefully chosen by the user.   


\textbf{Outline of the paper.} The organization of the paper is as follows. Section \ref{sec:model} describes the statistical model, its main assumption and introduces the notation that will be used along the paper. In Section \ref{sec:estimator_stopping_rule}, we introduce the $k$-NN estimator and explain how to compute the minimum discrepancy early stopping rule. Section \ref{sec:optimality_result} provides the main theoretical result that shows that the proposed rule is statistically optimal for some classes of functions (e.g., the well-known class of Lipschitz functions on a bounded domain). In Section \ref{sec:simulations}, one can find simulation results for the proposed stopping rule. To be precise, we compare this rule to the generalized cross-validation estimator, Akaike's AIC criterion, $V$--fold and Hold-out cross-validation stopping rules (\cite{arlot2010survey}) tested on some artificial and real-world data sets. Section \ref{sec:conc} concludes the paper. All the technical proofs are in the supplementary material.

\section{Statistical model, main assumption, and notation} \label{sec:model}

In the nonparametric regression setting, one works with a sample $(x_1, y_1), \ldots, \\ (x_n, y_n) \in \mathcal{X}^n \times \mathbb{R}^n$ that satisfies the statistical model
\begin{equation} \label{nonparametric_regression_model}
    y_i = f^*(x_i) + \varepsilon_i, \ \ i = 1,\ldots, n,
\end{equation}
where $f^*:\mathcal{X} \mapsto \mathbb{R}$, $\mathcal{X} \subseteq \mathbb{R}^d$, is a measurable function on some set $\mathcal{X}$, and $\{ \varepsilon_i \}_{i=1}^n$ are i.i.d. Gaussian noise variables $\mathcal{N}(0, \sigma^2)$ where the parameter $\sigma^2 > 0$ is fixed and unknown. 
The goal of the present paper is to estimate optimally the regression function $f^*$ (see \cite{tsybakov2008introduction, arlot2009data, wainwright2019high}). 

In the context of the \textit{fixed design} setting, when the covariates $x_1, \ldots, x_n$ are fixed, the performance of an estimator $\widehat{f}$ of $f^*$ is measured in terms of the so-called \textit{empirical norm} defined as 
\begin{equation}
    \lVert \widehat{f} - f^* \rVert_n^2 \coloneqq \frac{1}{n}\sum_{i=1}^n \left[ \widehat{f}(x_i) - f^*(x_i) \right]^2,  
\end{equation}
where $\lVert h \rVert_n \coloneqq \sqrt{1/n\sum_{i=1}^n h(x_i)^2}$ for any bounded on $\mathcal{X}$ function $h$. We denote the empirical norm as $L_2(\mathbb{P}_n)$. For each bounded over $\mathcal{X}$ functions $h_1, h_2$, $\langle h_1, h_2 \rangle_n$ denotes the related inner product defined as $\langle h_1, h_2 \rangle_n \coloneqq 1/n \sum_{i=1}^n h_1(x_i)h_2(x_i)$.
Further, $\mathbb{P}_{\varepsilon}$ and $\mathbb{E}_{\varepsilon}$ denote the probability and expectation with respect to $\{ \varepsilon_i \}_{i=1}^n$.

By contrast, when both the covariates $\{ x_i \}_{i=1}^n$ and the noise $\{ \varepsilon_i \}_{i=1}^n$ are random, we consider the so-called \textit{random design} setting. In this setting, the performance of an estimator $\widehat{f}$ of $f^*$ is measured in terms of the $L_2(\mathbb{P}_X)$ norm
\begin{equation}
    \lVert \widehat{f} - f^* \rVert_2^2 \coloneqq \mathbb{E}_{X} \left[ \left( \widehat{f}(X) - f^* (X) \right)^2 \right],
\end{equation}
where $\mathbb{P}_X$ denotes the probability distribution over $\{ x_i \}_{i=1}^n$. In what follows, $\mathbb{P}$ and $\mathbb{E}$ state for the probability and expectation with respect to the couples $\{ \left( x_i, y_i \right) \}_{i=1}^n$.

\textbf{Notation.}
Throughout the paper, $\lVert \cdot \rVert$ and $\langle \cdot, \cdot \rangle$ are the usual Euclidean norm and related inner product. $\lVert M \rVert_2$ and $\lVert M \rVert_{F}$ signify the spectral and Frobenius norms of the matrix $M \in \mathbb{R}^{n \times n}$, respectively. We denote the trace of the matrix $M$ by $\textnormal{tr}(M)$. In addition to that, $\mathbb{I}\left\{ \mathcal{E} \right\}$ is equal to $1$ if the probabilistic event $\mathcal{E}$ holds true, otherwise it is equal to $0$. For $a \geq 0$, we denote by $\floor*{a}$ the largest natural number that is smaller than or equal to $a$. We denote by $\ceil*{a}$ the smallest natural number that is greater than or equal to $a$. Along the paper, $I_n$ is the identity matrix of size $n \times n$.

We make the following assumption on the regression function $f^*$ introduced earlier in Eq. (\ref{nonparametric_regression_model}).
\begin{assumption}[Boundness of the r.f.] \label{assumption_boundness}
$f^*$ is bounded on $\mathcal{X}$, meaning that there exists a constant $\mathcal{M} > 0$ such that
\begin{equation}
    \left| f^*(x) \right| \leq \mathcal{M} \qquad \textnormal{ for all }x \in \mathcal{X}.
\end{equation}
\end{assumption}

Assumption \ref{assumption_boundness} is quite standard in the nonparametric regression literature (\cite{gyorfi2006distribution, zhao2021minimax}). In particular, Assumption \ref{assumption_boundness} holds when the set $\mathcal{X}$ is bounded, and the regression function $f^*$ is $L$-Lipschitz with some positive constant $L$ (cf. \cite{gyorfi2006distribution}).


Along the paper, we use the notation $c, c_1, C, \widetilde{c}, \widetilde{C}, \ldots$ to show that those positive numeric constants depend on $\sigma^2$ and $\mathcal{M}$ only, otherwise the dependence is said explicitly. The values of all the constants may change from line to line or even in the same line.

\section{$k$-NN estimator and minimum discrepancy stopping rule} \label{sec:estimator_stopping_rule}

\subsection{$k$-NN regression estimator}

Let us transform the initial nonparametric regression model (\ref{nonparametric_regression_model}) into its vector form 
\begin{equation} \label{nonparametric_regression_vector}
    Y \coloneqq \left[y_1, \ldots, y_n\right]^\top = F^* + \varepsilon \in \mathbb{R}^n, 
\end{equation}
where the vectors $F^* \coloneqq [f^*(x_1), \ldots, f^*(x_n)]^\top$ and $\varepsilon \coloneqq [\varepsilon_1, \ldots, \varepsilon_n]^\top$.

Define a $k$-nearest neighbor estimator $f^k$ of $f^*$ from (\ref{nonparametric_regression_model}) at the point $x_i, \ i = 1, \ldots, n,$ as
\begin{equation} \label{estimator}
    f^k(x_i) \coloneqq F^k_i = \frac{1}{k} \sum_{j \in \mathcal{N}_k(x_i)} y_j, \qquad k = 1, \ldots, n,
\end{equation}
where $\mathcal{N}_k(x_i)$ denotes the indices of the $k$ nearest neighbors of $x_i$ among $\{1, \ldots, n \}$ in the usual Euclidean norm in $\mathbb{R}^d$, where ties are broken at random. In words, in Eq. (\ref{estimator}) one weights by $1/k$ the response $y_j$ if $x_j$ is a $k$ nearest neighbor of $x_i$ measured in the Euclidean norm. Note that other adaptive metrics (instead of the Euclidean one) have been also considered in the literature (\cite[Chap. 14]{hastie2009elements}).

One can notice that the $k$-NN regression estimator (\ref{estimator}) belongs to the class of (local) linear estimators (\cite{2009arXiv0909.1884A, hastie2009elements}), meaning that $F^k \in \mathbb{R}^n$ estimates the vector $F^*$ as it follows.
\begin{equation} \label{main_estimator}
    F^k \coloneqq \left[ f^k(x_1), \ldots, f^k(x_n) \right]^\top = A_k Y,
\end{equation}
where $A_k \in \mathbb{R}^{n \times n}$ is the matrix described below.
\begin{equation} \label{a_k_matrix}
    \begin{cases}
    \forall 1 \leq i, j \leq n, \ \left( A_k \right)_{ij} \in \{0, 1/k\} \textnormal{ with } k \in \{1, \ldots, n \},\\
    \forall 1 \leq i \leq n, \ \left( A_k \right)_{ii} = 1/k \textnormal{ and } \sum_{j=1}^n \left( A_k \right)_{ij} = 1.
    \end{cases}    
\end{equation}
\vspace{0.05cm}
Saying differently, $(A_k)_{ij} = 1/k$ if $x_j$ is a $k$ nearest neighbor of $x_i$, otherwise $(A_k)_{ij} = 0, \ i, j \in \{1, \ldots, n\}$.

Define the mean-squared error (the risk error) of the estimator $f^k$ in the empirical norm as
\begin{equation} \label{mean-squared-error}
    \textnormal{MSE}(k) \coloneqq \mathbb{E}_{\varepsilon} \lVert f^k - f^* \rVert_n^2 = \frac{1}{n} \mathbb{E}_{\varepsilon} \sum_{i=1}^n \Big( \frac{1}{k}\sum_{j \in \mathcal{N}_k(i)}y_j - f^*(x_i) \Big)^2.
\end{equation}
Further, we will introduce the (squared) bias and variance of the functional estimator $f^k$ (see, e.g., \cite[Eq. (7)]{2009arXiv0909.1884A}),
\begin{equation}
    \textnormal{MSE}(k) = B^2(k) + V(k),
\end{equation}
where 
\begin{equation*}
    B^2(k) = \lVert (I_n - A_k)F^* \rVert_n^2, \qquad V(k) = \frac{\sigma^2}{n} \textnormal{tr}\left( A_k^\top A_k \right) = \frac{\sigma^2}{k}.
\end{equation*}
%

%
Thus, the variance term $\sigma^2/k$ is a decreasing function of $k$. Note that $B^2(1) = 0, \ V(1) = \sigma^2$, and $B^2(n) = (1 - 1/n)^2\lVert f^* \rVert_n^2, \ V(n) = \sigma^2/n$. Importantly, the bias term $B^2(k)$ can have \textit{arbitrary} behavior on the interval $[1, n]$. 

\vspace{0.05cm}

Ideally, we would like to minimize the mean-squared error (\ref{mean-squared-error}) as a function of $k$. However, since the bias term is not known (it contains the unknown regression function), one should introduce other quantities that will be related to the bias. In our case, this quantity will be the \textit{empirical risk} at $k$:
\begin{equation} \label{empirical_risk}
    R_k \coloneqq \lVert (I_n - A_k) Y \rVert_n^2.
\end{equation}

$R_k$ measures how well the estimator $f^k$ fits $Y$. Remark that $R_1 = 0$ (corresponds to the "overfitting" regime) and $R_n = (1 - 1/n)^2 \frac{1}{n} \sum_{i=1}^n y_i^2$ (corresponds to the "underfitting" regime), but there is no information about the monotonicity of $R_k$ on the interval $[1, n]$. 

Furthermore, some information about the bias is contained in the expectation (over the noise $\{ \varepsilon_i \}_{i=1}^n$) of the empirical risk. To be precise, for any $k \in \{ 1, \ldots, n\}$,
\begin{align} \label{exp_empirical_risk}
    \begin{split}
    \mathbb{E}_{\varepsilon}R_k &= \sigma^2 + B^2(k) - \frac{\sigma^2(2 \textnormal{tr}(A_k) - \textnormal{tr}(A_k^\top A_k))}{n}\\
    &= \sigma^2 + B^2(k) - V(k).
    \end{split}
\end{align}


Note that among all defined quantities, only the variance term $V(k)$ can be proved monotonic (without an additional assumption on the smoothness of $f^*$). Importantly, Fig \ref{fig:bvr} indicates that choosing $k = 6$ will provide the user with the global optimum of the risk (the mean-squared error) curve. Thus, for instance, it would be meaningless (according to the risk curve) to compute all the estimators $f^k$ (\ref{estimator}) for $k = 1, \ldots, 6$.

\begin{figure} 
\centering
\begin{tikzpicture}[scale=0.85]
\tikzstyle{every node}=[font=\normalsize]
\begin{axis}[
    title={},
    xlabel={Number of neighbors},
    ylabel={Value},
    xmin=1, xmax=12,
    ymin=0, ymax=0.02,
    xtick={1,3,5,7, 10},
    ytick={0.002, 0.006, 0.01, 0.02},
    legend pos=north west,
    ymajorgrids=true,
    xmajorgrids=true,
    grid style=dashed,
    legend style={nodes={scale=0.6, transform shape}}, 
]

\addplot[
    color=blue,
    mark=square,
    ]
    coordinates {
    (1,0)(2,6.29080327e-03)(3,8.47616891e-03)(4,1.08716416e-02)(5,1.12560012e-02)(6,1.14496210e-02)(7,1.37171039e-02)(8,1.57263962e-02)(9,1.83168233e-02)(10,1.78368289e-02)(11,1.94323757e-02)(12,1.98591401e-02)
    };
    \addlegendentry{Empirical risk}
\addplot[
    color=red,
    mark=+,
    ]
    coordinates {
    (1,0)(2,1.48163079e-03)(3,1.91217442e-03)(4,1.90325965e-03)(5,2.31901908e-03)(6,2.42449703e-03)(7,3.02981110e-03)(8,3.27183370e-03)(9,4.00873281e-03)(10,4.53268739e-03)(11,5.57342952e-03)(12,6.21032204e-03)
    };
    \addlegendentry{Bias}
\addplot[
    color=violet,
    mark=star,
    ]
    coordinates {
    (1,0.01)(2,0.005)(3,0.00333333)(4,0.0025)(5,0.002)(6,0.00166667)(7,0.00142857)(8,0.00125)(9,0.00111111)(10,0.001)(11,0.00090909)(12,0.00083333)
    };
    \addlegendentry{Variance}
\addplot[
    color=orange,
    mark=x,
    ]
    coordinates {
    (1,0.01)(2,0.00648163)(3,0.00524551)(4,0.00440326)(5,0.00431902)(6,0.00409116)(7,0.00445838)(8,0.00452183)(9,0.00511984)(10,0.00553269)(11,0.00648252)(12,0.00704366)
    };
    \addlegendentry{Risk}
\addplot[
    color=black,
    mark=*,
    ]
    coordinates {
    (1,0)(2,6.48163079e-03)(3,8.57884109e-03)(4,9.40325965e-03)(5,1.03190191e-02)(6,1.07578304e-02)(7,1.16012397e-02)(8,1.20218337e-02)(9,1.28976217e-02)(10,1.35326874e-02)(11,1.46643386e-02)(12,1.53769887e-02)
    };
    \addlegendentry{Exp. empirical risk}
\end{axis}
\end{tikzpicture}
\caption{Sq. bias, variance, risk and (expected) empirical risk behaviour.}
\label{fig:bvr}
\end{figure}
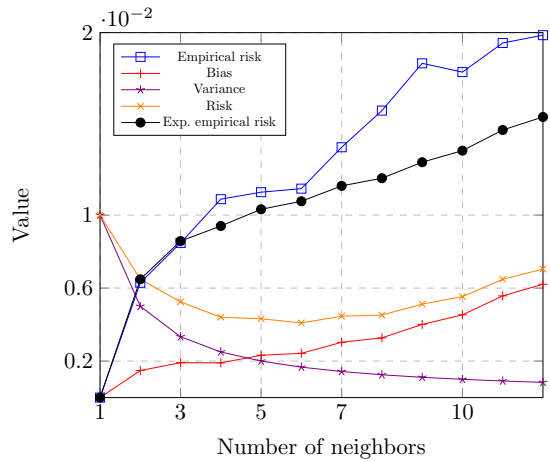


Our main concern is to design a data-driven strategy to choose $\widehat{k} \in \{1, \ldots, n \}$, which can be seen as a mapping from the data $\{ (x_i, y_i) \}_{i=1}^n$ to a positive integer so that the $L_2(\mathbb{P}_n)$ error $\lVert f^{\widehat{k}} - f^* \rVert_n^2$ (or the population $L_2(\mathbb{P}_X)$ error) is as small as possible. More precisely, the goal is to define a data-driven $\widehat{k}$ such that it satisfies the following non-asymptotic upper bound ("oracle-type inequality" \cite{wainwright2019high}):
\begin{equation} \label{or_type_ineq}
     \lVert f^{\widehat{k}} - f^* \rVert_n^2 \leq C_n \mathbb{E}_{\varepsilon} \lVert f^{k_{\textnormal{opt}}} - f^* \rVert_n^2 + r_n
\end{equation}
with high (exponential) probability over $\{ \varepsilon_i \}_{i=1}^n$, where $f^{k_{\textnormal{opt}}}$ is a minimax-optimal estimator of the regression function $f^* \in \mathcal{F}$, $\mathcal{F}$ is some a priori chosen function space. The leading constant $C_n$ should be bounded, and the remainder term $r_n$ is negligible (smaller) with respect to the "optimal risk error" $\mathbb{E}_{\varepsilon} \lVert f^{k_{\textnormal{opt}}} - f^* \rVert_n^2$. 

\subsection{Related work} \label{related_work}
The idea of choosing the hyperparameter $k$ from the data has been already considered in the literature. For example, the classical procedures such as generalized cross-validation (\cite{craven1978smoothing, li1987asymptotic, cao2006oracle}), penalized estimators (\cite{li1987asymptotic, mallows2000some, arlot2009data, arlot2009data_}), Lepskii principle (\cite{lepski1992problems, birge2001alternative}), and different cross-validation methods (\cite{arlot2010survey}) are popular choices for linear estimators. Let us consider them in more detail.

\textbf{Generalized CV (\cite{li1987asymptotic, cao2006oracle, hastie2009elements}).} This model selection method has been widely studied in the case of (kernel) ridge regression (\cite{craven1978smoothing}) and smoothing splines (\cite{cao2006oracle}). In particular, \cite{cao2006oracle} proved a non-asymptotic oracle inequality for the generalized CV estimator when the variance $\sigma^2$ is known. However, in a more general case, GCV estimates $\sigma^2$ implicitly, which is an advantage of the method. In addition to that, GCV for $k$-NN regression is proved by \cite{li1987asymptotic} to be an asymptotically optimal model selection criterion (i.e., $\lVert f^{k_{\textnormal{GCV}}} - f^* \rVert_n^2 \ / \ \underset{k}{\inf}\lVert f^{k} - f^* \rVert_n^2 \to 1$ in probability when $n \to + \infty$) under the assumption $\lVert A_k \rVert_2 \leq c, \ k = 1, \ldots, n,$ for some constant $c$. It is worth mentioning that generalized cross-validation provides an approximation to the so-called leave-one-out cross-validation (\cite{arlot2010survey, celisse2018theoretical}), which is an exhaustive model selection procedure. In the case of GCV, if the nearest neighbors' matrices are already precomputed, its computational time is $\mathcal{O}\left( n^3 \right)$ elementary operations. The GCV strategy will be later considered in our simulations (see Section \ref{sec:simulations}).

\textbf{Lepskii principle (\cite{lepski1992problems, lepskii1992asymptotically, lepskii1993asymptotically})}. Lepskii designed a strategy for choosing a data-dependent parameter which is based on pairwise comparisons between estimators. This method was first developed for the Gaussian white noise model (\cite{lepski1992problems, birge2001alternative}) and afterwards extended to the kernel bandwidth selection (\cite{lepski1997optimal}) and other frameworks (see \cite{goldenshluger2013general} for a modification of the method). While enjoying optimal theoretical properties, the method suffers from tuning issues and high computation cost due to pairwise comparison of all (kernel) estimators. In particular, Lepskii's method involves a parameter that needs to be chosen, and in practice this method is very sensitive to this choice (e.g., parameter $D_3$ and noise level $\sigma^2$ in \cite{lepski1997optimal}). The computation cost of Lepskii principle is particularly high in the multivariate setting (see \cite{bertin2016adaptive}). To the best of our knowledge, there is no work done on applying the Lepskii principle to the $k$-NN regression estimator, specifically.

\textbf{Penalized estimators} date back to the work on AIC (\cite{akaike1998information}) or  Mallow's $C_p$ (\cite{mallows2000some}) criteria, where a penalty proportional to the dimension of the model is added to the quadratic loss (i.e., \textit{the empirical risk error} in our notation (\ref{empirical_risk})) when the noise level $\sigma^2$ is known. As it was for the GCV strategy, the asymptotic computational time of AIC and Mallow's $C_p$ is $\mathcal{O}\left( n^3 \right)$. After that, a new approach was developed by \cite{birge2007minimal}, where the authors introduced the so-called "slope heuristics" for projection matrices. This notion was based on the introduction of the penalty $\textnormal{pen}(k) = K \textnormal{tr}(A_k)$, where $\textnormal{tr}(A_k)$ is the dimension of the model, and $K$ is a constant that can depend on $\sigma^2$, in particular. It appeared that there exists a constant $K_{\textnormal{min}}$ such that $2 K_{\min} \textnormal{tr}(A_k)$ yields an asymptotically optimal model selection procedure. It gives rise to some strategies for the estimation of constant $K_{\textnormal{min}}$ from the data, as it was done, for instance, by \cite{2009arXiv0909.1884A} for some linear estimators when $\sigma^2$ is unknown.

\textbf{Cross-validation methods (\cite{arlot2010survey}).} These model selection methods are the most used in practice. Compared to generalized cross-validation, for instance, $V$--fold cross-validation method (\cite{geisser1975predictive, arlot2010survey}) incurs a large computational cost (with $V$, which is not too small). To be precise, the $V$--fold cross-validation requires the model selection procedure to be performed $V$ times for each value of $k \in \{1, \ldots, n\}$. Another alternative could be the Hold-out method (\cite{wegkamp2003model, arlot2010survey}), which consists in randomly splitting the data into two parts for each value $k \in \{1, \ldots, n\}$:  one is dedicated for training the estimator (\ref{estimator}) and the other one is dedicated for testing (see, e.g., Section \ref{sec:simulations} for more details in a simulated example).

\subsection{Bias-variance trade-off and minimum discrepancy principle rule}

We are at the point to define our first "reference rule". Based on the nonparametric statistics literature (\cite{wasserman2006all, tsybakov2008introduction}), the bias-variance trade-off (with a noise-level estimation correction) usually provides an optimal estimator in the minimax sense:
\begin{equation} \label{k_star}
    k^* = \inf \left\{ k \in \{ 1, \ldots, n \} \mid B^2(k) \geq V(k) + 2\mathbb{E}_{\varepsilon}R_2 - \sigma^2  \right\}.
\end{equation}

The bias-variance trade-off stopping rule $k^*$ does not always exist due to arbitrary behavior of the bias term $B^2(k)$. Thus, if no such $k^*$ exists, set $k^* = n$. Notice that the stopping rule $k^*$ is \textit{not computable} in practice, since it depends on the unknown bias $B^2(k)$ and noise level $\sigma^2$. Nevertheless, we can create a data-driven version of $k^*$ by means of the empirical risk $R_k$ from Eq. (\ref{empirical_risk}). 

Eq. (\ref{exp_empirical_risk}) gives us that the event $\{ B^2(k) \geq V(k) + 2\mathbb{E}_{\varepsilon}R_2 - \sigma^2 \}$ is equivalent to the event $\{ \mathbb{E}_{\varepsilon}R_k \geq 2\mathbb{E}_{\varepsilon}R_2 \}$, so we conclude that $k^* = \inf \{ k \in \{ 1, \ldots, n \} \mid \mathbb{E}_{\varepsilon}R_k \geq 2\mathbb{E}_{\varepsilon}R_2 \}$. This gives rise to an estimator of $k^*$ that we denote as $k^{\tau}$. This stopping rule is called \textit{the minimum discrepancy principle} stopping rule and is defined as
\begin{equation} \label{k_tau}
    k^{\tau} = \sup \left\{ k \in \{ 1, \ldots, n \} \mid R_k \leq 2R_2 \right\}. 
\end{equation}

\textbf{Remarks.} Note that in Eq. (\ref{k_tau}), we introduced a supremum instead of the infimum from Eq. (\ref{k_star}). That was done on purpose since there could be several points of the bias-variance trade-off, and the bias (and the empirical risk) could behave badly in the area "in-between". To calculate $k^{\tau}$, the user should, first, compute the empirical risk $R_k$ at $k = n$ (thus, the matrix $A_n$ of $n$ nearest neighbors). Then one needs to decrease $k$ until the event $\{ R_k \leq 2R_2 \}$ holds. In Eq. (\ref{k_tau}), $2R_2$ serves as an estimator of the noise level $\sigma^2$ since $\mathbb{E}_{\varepsilon}\left(2R_2 - \sigma^2 \right) = 2 B^2(2)$ that should be small. It is worth mentioning that it is not necessary to compute explicitly all the matrices $A_k, \ k = n, \ n-1, \ldots$, since, for instance, the matrix $A_{n-1}$ could be easily derived from the matrix $A_n$ (assuming that one has already arranged the neighbors and removed the $n^{\text{th}}$ neighbors from the matrix $A_n$), i.e., 
\begin{equation} \label{iteration_for_matrix}
    [A_{n-1}]_{ij} = \frac{n}{n-1}[A_n]_{ij}, \ \forall i, j \in \{1, \ldots, n\}.
\end{equation}
It is one of the main computational advantages of the proposed rule (\ref{k_tau}). For more details on the efficient computation of the nearest neighbors' matrices, see, e.g., \cite{bentley1975multidimensional, omohundro1989five}. In addition to all of that, we emphasize that the definition (\ref{k_tau}) of $k^{\tau}$ does not require the knowledge of the constant $\mathcal{M}$ from Assumption \ref{assumption_boundness}, and $k^{\tau}$ does not require computing the empirical risk $R_k$ for all values $k = 1, \ldots, n$, while it is the case, for instance, for generalized cross-validation or Mallow's $C_p$ (see Section \ref{sec:simulations}). Moreover, we should point out that the stopping rule (\ref{k_tau}) doesn't depend on the noise level $\sigma^2$ as for the AIC or Mallow's $C_p$ criteria (\cite{akaike1974new, mallows2000some, hastie2009elements}). Unlike Lepskii principle, (\cite{arlot2009data}) or (\cite{zhao2021minimax}), the MDP rule does not have parameters to tune and involves computing only a fraction of the estimators. Regarding the asymptotic computational time of $k^{\tau}$, if the nearest neighbors' matrices are already computed, it is of the order $\mathcal{O}\left( n^2 \left( n - k^{\tau} \right) \right)$, while it is of the order $\mathcal{O}\left( n^3 \right)$ for the AIC/Mallow's $C_p$ criteria or GCV.

\vspace{0.15cm}

There is a large amount of literature (\cite{engl1996regularization, bauer2007regularization, blanchard2018optimal, blanchard2018early, celisse2021analyzing}) on the minimum discrepancy principle for spectral filter algorithms such as gradient descent, ridge regularization, and spectral cut-off regression, and providing an exhaustive review on this strategy is out of the scope of the paper (e.g., \cite{blanchard2018optimal, celisse2021analyzing} provide a thorough review). We should emphasize that the minimum discrepancy principle determines the first time a learning algorithm starts to fit noise, which is measured by $2 R_2$ in the present context. 

\vspace{0.15cm}
In what follows, we show that for a quite large class of regression functions, $k^{\tau}$ is optimal in the sense of Ineq. (\ref{or_type_ineq}).

\section{Theoretical optimality result} \label{sec:optimality_result} 
Let us start to describe the main theoretical result of the present paper. The following theorem applies to the estimator defined in Eq. (\ref{main_estimator}). 
\begin{theorem} (Upper bound on the empirical norm). \label{main_th}
Under Assumption \ref{assumption_boundness}, for arbitrary $v_1 \geq 0$ and $v_2 \geq \frac{4 \mathcal{M}^2}{\sigma^2}$,
    \begin{equation} \label{main_ineq}
        \lVert f^{k^{\tau}} - f^* \rVert_n^2 \leq 8 V(k^*) + C_1 \left( \frac{v_1\sigma^2}{n} + \sqrt{\frac{v_1 \sigma^4}{n}} + \sqrt{\frac{v_1\sigma^2}{n}} \right) + C_2 \sqrt{\frac{\log n}{n}} + 2 ( 2\mathbb{E}_{\varepsilon}R_2 - \sigma^2 )
\end{equation}
    with probability at least $1 - 18 \exp \left(- v_1 \right) - 5\exp \left( - \min \left( 1, \frac{\sigma^2}{128 \mathcal{M}^2} \right) n v_2^2 \left( 1 - \frac{1}{n^{1.5}2^{n/2}} \right)^2 \right) - \frac{5 (v_2 + v_2^2)}{\sqrt{n \left( n/2 - 1 \right)}}$, where constant $C_1 > 0$ can depend on $\mathcal{M}, \sigma^2$ and constant $C_2 > 0$ can depend on $\sigma^2$.
    
Moreover, if $k^*$ from Eq. (\ref{k_star}) exists, then for arbitrary $v_1 \geq 0$, 
\begin{equation} \label{corollary_from_the_main}
    \lVert f^{k^{\tau}} - f^* \rVert_n^2 \leq \underbrace{4 \ \textnormal{MSE}(k^*)}_{\textnormal{Main term}} + \underbrace{C_1 \left( \frac{v_1\sigma^2}{n} + \sqrt{\frac{v_1 \sigma^2 }{n}} + \sqrt{\frac{v_1 \sigma^4}{n}} \right) + C_2 \sqrt{\frac{\log n}{n}} + 2 (\sigma^2 - 2\mathbb{E}_{\varepsilon}R_2) }_{\textnormal{Remainder term}}
\end{equation}
with the same probability, where constants $C_1, C_2$ are from Ineq. (\ref{main_ineq}).
\end{theorem}

\begin{proof}
The full proof is deferred to the supplementary material. Let us provide a sketch of the proof here.

The main ingredients of the proof are two deviation inequalities (cf. Corollary \ref{corollary_for_variance} and Lemma \ref{deviation_bias_lemma} in the supplementary material): for any $t_1 \geq 0$ and $t_2 \geq \frac{4 \mathcal{M}^2}{\sigma^2}$,
\begin{equation}
    \mathbb{P}_{\varepsilon} \left( V(k^{\tau}) > 2 V(k^*) + t_1 \right) \leq 4 \exp \left( -c n \min \left( \frac{t_1^2}{\sigma^2}, \frac{t_1^2}{\sigma^4}, \frac{t_1}{\sigma^2} \right) \right),
\end{equation}
and 
\begin{equation} \label{bias_k_tau_variance_k_star}
    B^2(k^{\tau}) < 2 V(k^*) + 2\mathbb{E}_{\varepsilon}R_2 - \sigma^2 + 2t_1,
\end{equation}
where Ineq. (\ref{bias_k_tau_variance_k_star}) holds with probability at least $1 - 12 \exp \left( -c n \min \left( \frac{t_1^2}{\sigma^2}, \frac{t_1^2}{\sigma^4}, \frac{t_1}{\sigma^2} \right) \right) - 5 \exp \left( -\min \left( 1, \frac{\sigma^2}{128 \mathcal{M}^2}  \right) n t_2^2 \left( 1 - \frac{1}{n^{1.5}2^{n/2}} \right)^2 \right) - \frac{5 (t_2 + t_2^2)}{\sqrt{n \left( n/2 - 1 \right)}}$.

After that, one can split the $L_2(\mathbb{P}_n)$-error at $k^{\tau}$ into two parts:
\begin{equation*}
    \lVert f^{k^{\tau}} - f^* \rVert_n^2 \leq 2 B^2(k^{\tau}) + 2 \lVert A_{k^{\tau}}\varepsilon \rVert_n^2.   
\end{equation*}

It is sufficient to derive high probability control of $\underset{k}{\sup} \left| \lVert A_{k}\varepsilon \rVert_n^2 - V(k)   \right|$ for $k = 1, \ldots, n$ (see Appendix \ref{var_control} in the supplementary material). That was the reason why the term $\mathcal{O}\left(\sqrt{\frac{\log n}{n}}\right)$ appeared in Eq. (\ref{main_ineq}).

Finally, one can apply $V(k^*) \leq \frac{1}{2} \left( \textnormal{MSE}(k^*) + \sigma^2 - 2\mathbb{E}_{\varepsilon}R_2 \right)$, if $k^*$ exists, and set $v_1 = c n \min \left( \frac{t_1}{\sigma^2}, \frac{t_1^2}{\sigma^4}, \frac{t_1^2}{\sigma^2} \right)$. The claim follows.

\end{proof}

In order to gain some intuition of the claim of Theorem \ref{main_th}, let us make some comments.

First of all, Ineq. (\ref{corollary_from_the_main}) is non-asymptotic, meaning that it holds for any sample size $n \geq 3$. Second, Ineq. (\ref{corollary_from_the_main}) holds "with high probability", which is a stronger result than in expectation since (\cite{li1987asymptotic}) there are model selection procedures that are asymptotically optimal (when $n \to +\infty$) in expectation but not with high probability. 

Third, the main term in Ineq. (\ref{corollary_from_the_main}) is the risk error at the bias-variance trade-off times $4$. Ideally, one should rather introduce $\underset{k=1, \ldots, n}{\inf} \left[ \mathbb{E}_{\varepsilon} \lVert f^k - f^* \rVert_n^2 \right]$ and compare $\lVert f^{k^{\tau}} - f^* \rVert_n^2$ to it. However, to the best of our knowledge, a smoothness assumption is needed to connect the bias-variance trade-off risk and the oracle (minimum) risk. That was the reason to keep the main term as it was stated. Fourth, the right hand side term of Ineq. (\ref{corollary_from_the_main}) is of the order $\mathcal{O}\left(\sqrt{\frac{\log n}{n}}\right) + ( 2\mathbb{E}_{\varepsilon}R_2 - \sigma^2 )$. Notice that the rate $\mathcal{O}\left( \sqrt{\frac{\log n}{n}} \right)$ for the remainder term was achieved in (\cite{azadkia2019optimal}) but in terms of the expectation over the noise. The term $2\mathbb{E}_{\varepsilon}R_2 - \sigma^2 $ corresponds to the noise level estimation. 


A natural question would be to understand if the rate $\mathcal{O}\left(\sqrt{\frac{\log n}{n}}\right) + 2 \mathbb{E}_{\varepsilon} R_2 - \sigma^2$ in Ineq. (\ref{main_ineq}) is sufficiently fast. In order to do that, one should precise the function space $\mathcal{F}$ where $f^*$ lies in. In what follows, we will mention one famous example (among others) of a such function space $\mathcal{F}$.
\begin{example}
Consider the class of functions 
\begin{equation} \label{lipschitz}
    \mathcal{F}_{\textnormal{Lip}}(L) \coloneqq \left\{ f: [0, 1]^d \mapsto \mathbb{R} \mid f(0) = 0, \ f \textnormal{ is } L-\textnormal{Lipschitz} \right\},
\end{equation}
where $f$ is $L$-Lipschitz means that $|f(x) - f(x^\prime)| \leq L \lVert x - x^{\prime} \rVert$ for all $x, x^{\prime} \in [0, 1]^d$. In this case (see, e.g., \cite[Theorem 3.2]{gyorfi2006distribution} with $p = 1$), 
\begin{equation} \label{minimax_lipschitz}
    \underset{\widehat{f}}{\inf}\underset{f^* \in \mathcal{F}_{\textnormal{Lip}}(L)}{\sup}\mathbb{E} \lVert \widehat{f} - f^* \rVert_2^2 \geq c_l n^{-\frac{2}{2 + d}},
\end{equation}
for some positive constant $c_l$ that depends on $d, L$, and $\sigma^2$, for any measurable of the input data $\widehat{f}$.
\end{example}
Therefore, for the class of $L$-Lipshitz functions, the rate $\mathcal{O}(\sqrt{\log n / n})$ is faster than the minimax-optimal rate $\mathcal{O}\left(n^{-\frac{2}{2+d}}\right)$ for any $d > 2$. To prove the next result for the optimality of $k^\tau$, one needs to transfer Ineq. (\ref{main_ineq}), derived in the empirical $L_2\left(\mathbb{P}_n\right)$ norm, to the population $L_2\left(\mathbb{P}_X\right)$ norm via the Hoeffding inequality (cf. Lemma \ref{hoeffding_concentration} in the supplementary material). 
%
%
%
%
%
Following this argument, we summarize our findings in the corollary below.

\begin{corollary}(Upper bound on risk in the population norm). \label{main_corollary}
    Under the $L$-Lipschitz condition (\ref{lipschitz}) on the regression function $f^*$, early stopping rule $k^{\tau}$ from Eq. (\ref{k_tau}) satisfies
    \begin{equation}
        \mathbb{E} \lVert f^{k^{\tau}} - f^* \rVert_2^2 \leq c_{u,1} n^{-\frac{2}{2+d}} + c_{u,2}\exp\left( - c \min \left( n^{\frac{d-2}{d+2}}, n^{\frac{2}{d+2}}, n \left( 1 - \frac{1}{n^{1.5}2^{n/2}} \right)^2 \right) \right),
    \end{equation}
    where positive constants $c_{u,1}, c_{u, 2}, c$ can depend on $d, \sigma^2$, and $L$; $d > 2$ and $n \geq 3$.
\end{corollary}
The proof is deferred to the supplementary material (see Appendix \ref{section_proof_corollary}). 

\vspace{0.1cm}

Therefore, the function estimator $f^{k^{\tau}}$ achieves (up to a constant and for $n$ large enough) the minimax bound in the population norm presented in Eq. (\ref{minimax_lipschitz}), thus non-improvable in general for the class of Lipschitz functions on a bounded domain. 
\section{Empirical comparison to other model selection rules} \label{sec:simulations}

The present section aims at comparing the practical behavior of our stopping rule $k^{\tau}$ from Eq. (\ref{k_tau}) with other existing and the most used-in-practice model selection rules. We split the section into three parts: Subsection \ref{description_early_stopping_rules} defines the competitive stopping rules. Subsection \ref{artificial_data} presents experiments on some artificial data sets, while Subsection \ref{real_data} presents experiments on some real data sets. 
\subsection{Description of the model selection rules to compare} \label{description_early_stopping_rules}
In what follows, we will briefly describe five competitive model selection rules. 
\subsection*{Akaike's AIC criterion}
The Akaike's information criterion (\cite{akaike1974new, hastie2009elements}) estimates the risk error by means of a log-likelihood loss function. In the case of $k$-NN regression with Gaussian noise, the maximum likelihood and least-squares are essentially the same things. This gives the AIC criterion as
\begin{equation*}
    R_{\textnormal{AIC}}(f^k) = \frac{1}{n \widehat{\sigma}^2} \left( \lVert Y - A_k Y \rVert^2 + 2 \textnormal{tr}(A_k)\widehat{\sigma}^2 \right), \qquad k = 1, \ldots, n,
\end{equation*}
where $\widehat{\sigma}^2$ is an estimator of $\sigma^2$ obtained from a low-bias model (i.e. with $k = 2$). Using this criterion,
we adjust the training error by a factor proportional to the "degree of freedom". Then, the AIC choice for the optimal $k$ is
\begin{equation} \label{k_aic}
    k_{\textnormal{AIC}} \coloneqq \underset{k = 2, \ldots, n}{\textnormal{argmin}} \left\{ R_{\textnormal{AIC}}(f^k) \right\} - 1.
\end{equation}
Notice that in the mentioned case, AIC criterion is equivalent to Mallow's $C_p$ criterion (\cite{mallows2000some}). AIC criterion has been widely criticized in the literature, especially for the constant $2$ in the definition and/or its asymptotic nature. That is why some authors proposed corrections for the criterion (\cite{schwarz1978estimating, yang1999model}). Nevertheless, AIC, $C_p$, and other related criteria have been proved to satisfy non-asymptotic oracle-type inequalities (see \cite{birge2007minimal} and references therein). Notice that the computational time of the AIC criterion is $\mathcal{O}\left( n^3 \right)$.
\subsection*{Generalized cross-validation.}
The generalized (GCV) cross-validation strategy (\cite{craven1978smoothing, arlot2010survey}) was introduced in least-squares regression as a rotation-invariant version of the leave-one-out cross-validation procedure. The GCV estimator of the risk error of the linear estimator $A_k Y, \ k = 1, \ldots, n,$ is defined via
\begin{equation*}
    R_{\textnormal{GCV}}(f^k) = \frac{n^{-1}\lVert Y - A_k Y \rVert^2}{(1 - n^{-1}\textnormal{tr}(A_k))^2},    
\end{equation*}
The final model selection rule is
\begin{equation} \label{k_1_out}
    k_{\text{GCV}} \coloneqq \underset{k = 2, \ldots, n}{\textnormal{argmin}} \left\{ R_{\textnormal{GCV}}(f^k) \right\} - 1.
\end{equation}
GCV should be close to the AIC model selection procedure when the sample size $n$ is large. The asymptotic optimality of GCV, meaning that $\lVert f^{k_{\textnormal{GCV}}} - f^* \rVert_n^2 \ / \ \underset{k}{\inf}\lVert f^{k} - f^* \rVert_n^2 \to 1$ in probability, has been proved for the $k$-NN estimator in (\cite{li1987asymptotic}) under mild assumptions. As for the AIC criterion, the computational time of generalized cross-validation is $\mathcal{O}\left( n^3 \right)$ elementary operations.
\subsection*{Hold-out cross-validation stopping rule.}
The Hold-out cross-validation strategy (\cite{geisser1975predictive, arlot2010survey}) is described as follows. The data $\{ x_i, y_i \}_{i=1}^n$ are randomly split into two parts of equal size: the training sample $S_{\textnormal{train}} = \{ x_{\textnormal{train}}, y_{\textnormal{train}} \}$ and the test sample $S_{\textnormal{test}} = \{ x_{\textnormal{test}}, y_{\textnormal{test}} \}$ so that the training and test samples represent a half of the whole data set $\approx n/2$. For each $k = 1, \ldots, n$, one trains the $k$-NN estimator (\ref{estimator}) and evaluates its performance by $R_{\text{HO}}(f^k) = \frac{1}{n}\sum_{i \in S_{\textnormal{test}}}(f^k(x_i)- y_i)^2$, where $f^k(x_i)$ denotes the output of the algorithm trained for $k$ and evaluated at the point $x_i \in x_{\textnormal{test}}$. Then, the Hold-out CV stopping rule is defined via
\begin{equation} \label{k_ho}
    k_{\text{HO}} \coloneqq \underset{k = 2, \ldots, n}{\textnormal{argmin}} \left\{ R_{\text{HO}}(f^k) \right\} - 1. 
\end{equation}
The main inconvenience of this model selection rule is the fact that a part of the data is lost, which increases the risk error. Besides that, the Hold-out strategy is not stable (\cite{arlot2010survey}), which often requires some aggregation of it. The (asymptotic) computational time of the Hold-out strategy is $\mathcal{O}\left( n^3 \right)$ elementary operations.
\subsection*{$V$--fold cross-validation}
$V$--fold cross-validation is certainly the most used cross-validation procedure: the data $\{ (x_i, y_i) \}_{i=1}^n$ are randomly split into $V=5$ equal sized blocks, and at each round (among the $V$ ones), $V-1$ blocks are devoted to training $S_{\text{train}} = (x_{\text{train}}, y_{\text{train}})$, and the remaining one is used for the evaluation $S_{\text{test}} = (x_{\text{test}}, y_{\text{test}})$. The risk error of the $k$-NN estimator is estimated by $R_{\text{VFCV}}(f^k) = \frac{1}{V}\sum_{j=1}^{V} \frac{1}{n/V}\sum_{i \in S_{\text{test}}(j)}\left( f^k(x_i) - y_i \right)^2$, where $f^k(x_i)$ denotes the output of the algorithm trained for $k$ and evaluated at the point $x_i \in S_{\text{test}}(j)$, thus 
\begin{equation} \label{k_vfcv}
    k_{\text{VFCV}} \coloneqq \underset{k = 2, \ldots, n}{\textnormal{argmin}} \left\{ R_{\text{VFCV}}(f^k) \right\} - 1.
\end{equation}
$V$--fold cross-validation is a more computationally tractable solution than other splitting-based model selection methods, such as the leave-one-out (\cite{azadkia2019optimal}) or leave-$p$-out (\cite{hastie2009elements, arlot2010survey}). Usually, the optimal $V$ is equal to $5$ or $10$ due to the fact that the statistical error does not increase a lot for larger values of $V$ whereas averaging over more than $10$ folds becomes infeasible. To the best of our knowledge, there are no theoretical results for the $V$--fold cross validation model selection strategy with the $k$-NN regression estimator.

\subsection*{Theoretical bias-variance trade-off stopping rule}

The fourth stopping rule is the one introduced in Eq. (\ref{k_star}). This stopping rule is the classical bias-variance trade-off stopping rule (with a noise-level estimation correction) that provides minimax-optimal rates (see the monographs \cite{wasserman2006all, tsybakov2008introduction}):
\begin{equation}
    k^* = \inf \{ k \in \{ 1, \ldots, n \} \mid B^2(k) \geq V(k) +  2\mathbb{E}_{\varepsilon}R_2 - \sigma^2 \}. 
\end{equation}
The stopping rule $k^*$ is introduced for comparison purposes only because it \textit{cannot be computed} in practice. One can say that this stopping rule is minimax-optimal if $f^*$ belongs, for instance, to the class of Lipschitz functions on a bounded domain (\ref{lipschitz}). Therefore, it could serve as a (lower bound) reference in the present simulated experiments with artificial data.


\subsection{Artificial data} \label{artificial_data}

First, the goal is to perform some simulated experiments (a comparison of mentioned stopping rules) on artificial data.

\vspace{0.2cm}

\textbf{Description of the simulation design}

The data in this case is generated according to the regression model $y_j = f^*(x_j) + \varepsilon_j$, where $\varepsilon_j \overset{\text{i.i.d.}}{\sim} \mathcal{N}(0, \sigma^2)$ (Gaussian),  $j = 1, \ldots, n$. We choose the uniform covariates $x_j \overset{\text{i.i.d.}}{\sim} \mathbb{U}[0, 1]^3$, $j = 1, \ldots, n,$ and $\sigma = 0.15$. Consider two regression functions with different smoothness: a "smooth" $f_1^*(x) = 1.5 \cdot \left[ \lVert x - 0.5 \rVert / \sqrt{3} - 0.5 \right]$ and a "sinus" $f_2^*(x) = 1.5 \cdot \sin (\lVert x \rVert / \sqrt{3})$ for any $x \in [0, 1]^3$. Notice that both functions belong to the class of Lipschitz functions (\ref{lipschitz}) on $[0, 1]^3$. The sample size $n$ varies from $50$ to $250$.

The $k$-NN algorithm (\ref{estimator}) is trained first for $k = \floor*{\sqrt{n}}$, then we decrease the value of $k$ until $k = 1$ such that at each step of the iteration procedure we increase the variance of the $k$-NN estimator $V(k) = \sigma^2/k$ (cf. Fig. \ref{fig:bvr}). In other words, the model becomes more complex successively due to the increase of its "degree of freedom" measured by $\textnormal{tr}(A_k) = n/k$. If the condition in Eq. (\ref{k_tau}) is satisfied, the learning process is stopped, and it outputs the stopping rule $k^{\tau}$. 

The performance of the stopping rules is measured in terms of the empirical $L_2(\mathbb{P}_n)$ norm $\lVert f^k - f^* \rVert_n^2$ averaged over $N = 1000$ repetitions (over the noise $\{ \varepsilon_j \}_{j=1}^n$).  


\begin{figure}
\begin{subfigure}[b]{0.45\textwidth}
\begin{tikzpicture}[scale=0.7]
\tikzstyle{every node}=[font=\normalsize]
\begin{axis}[
    title={$k$-NN regressor, $\sigma = 0.15$, smooth r.f.},
    xlabel={Sample size},
    ylabel={Average loss},
    xmin=35, xmax=265,
    ymin=0.003, ymax=0.0125,
    ytick={0.005,0.007,0.009},
    xtick={50, 100, 150, 200, 250},
    legend pos=north east,
    ymajorgrids=true,
    xmajorgrids=true,
    grid style=dashed,
]
\addplot+[
    color=green,
    mark=star,
    ][error bars/.cd,y dir=both, y explicit]
    coordinates {
    (50,0.00760317)+-(0, 0.00188864)(80,0.00742744)+-(0, 0.00154863)(100,0.00580828)+-(0,0.00145157)(160,0.00517477)+-(0, 0.00100247)(200,0.00501453)+-(0, 0.00084924)(250,0.00442324)+-(0, 0.00070415)
    };
    \addlegendentry{$k^{\tau}$}
\addplot+[
    color=red,
    mark=+,
    ][error bars/.cd,y dir=both, y explicit]
    coordinates {
    (50, 0.00910639)+-(0, 0.00323673)(80, 0.00796068)+-(0, 0.00199004)(100, 0.00733271)+-(0, 0.00238249)(160, 0.0055938)+-(0, 0.00133856)(200, 0.00563115)+-(0, 0.00150105)(250,0.0049576)+-(0, 0.0012489)
    };
    \addlegendentry{$k_{\textnormal{HO}}$}
\addplot+[
    color=black,
    mark=x,
    ][error bars/.cd,y dir=both, y explicit]
    coordinates {
    (50, 0.00761403)+-(0, 0.00203475)(80, 0.0074194)+-(0, 0.00163831)(100, 0.00568851)+-(0, 0.001327)(160, 0.00512856)+-(0, 0.00097094)(200, 0.00491604)+-(0, 0.00085104)(250,0.00433407)+-(0, 0.00070607)
    };
    \addlegendentry{$k_{\textnormal{GCV}}$}
\addplot+[
    color=blue,
    mark=star,
    ][error bars/.cd,y dir=both, y explicit]
    coordinates {
    (50, 0.0076303)+-(0, 0.00169646)(80, 0.00716394)+-(0, 0.00142898)(100, 0.00568271)+-(0, 0.00119752)(160, 0.00510824)+-(0, 0.00089136)(200, 0.00475091)+-(0, 0.0007806)(250,0.00418769)+-(0, 0.00065183)
    };
    \addlegendentry{$k^{*}$}
\end{axis}
\end{tikzpicture}
\caption{}
\label{smooth}
\end{subfigure}
\qquad
\begin{subfigure}[b]{0.45\textwidth}
\begin{tikzpicture}[scale=0.7]
\tikzstyle{every node}=[font=\normalsize]
\begin{axis}[
    title={$k$-NN regressor, $\sigma = 0.15$, sinus r.f.},
    xlabel={Sample size},
    ylabel={Average loss},
    xmin=35, xmax=265,
    ymin=0.0018, ymax=0.0125,
    xtick={50,100,150,200,250},
    ytick={0.003,0.005, 0.008},
    legend pos=north east,
    ymajorgrids=true,
    xmajorgrids=true,
    grid style=dashed,
]
\addplot+[
    color=green,
    mark=star,
    ][error bars/.cd,y dir=both, y explicit]
    coordinates {
    (50,0.00809793)+-(0, 0.00217486)(80,0.00632696)+-(0, 0.00149826)(100,0.00547978)+-(0, 0.00151349)(160,0.00387815)+-(0, 0.00093082)(200,0.00379687)+-(0, 0.00083139)(250,0.00293352)+-(0, 0.0007208)
    };
    \addlegendentry{$k^{\tau}$}
\addplot+[
    color=red,
    mark=+,
    ][error bars/.cd,y dir=both, y explicit]
    coordinates {
    (50, 0.00905597)+-(0, 0.00307869)(80, 0.00833658)+-(0, 0.00281476)(100, 0.00665805)+-(0, 0.00238971)(160, 0.00458288)+-(0, 0.00170099)(200, 0.0044963)+-(0, 0.00151056)(250,0.00327392)+-(0, 0.00104712)
    };
    \addlegendentry{$k_{\textnormal{HO}}$}
\addplot+[
    color=black,
    mark=x,
    ][error bars/.cd,y dir=both, y explicit]
    coordinates {
    (50, 0.00814592)+-(0, 0.00219079)(80, 0.00643671)+-(0, 0.00157249)(100, 0.0054034)+-(0, 0.00140922)(160, 0.00387545)+-(0, 0.00086278)(200, 0.00377195)+-(0, 0.00082525)(250,0.00294051)+-(0, 0.00062037)
    };
    \addlegendentry{$k_{\textnormal{GCV}}$}
\addplot+[
    color=blue,
    mark=star,
    ][error bars/.cd,y dir=both, y explicit]
    coordinates {
    (50, 0.00782676)+-(0, 0.00212944)(80, 0.00623797)+-(0, 0.00141762)(100, 0.00522841)+-(0, 0.00127438)(160, 0.00371328)+-(0, 0.00078728)(200, 0.00360601)+-(0, 0.0007269)(250,0.00281642)+-(0, 0.00054082)
    };
    \addlegendentry{$k^{*}$}
\end{axis}
\end{tikzpicture}
\caption{}
\label{sinus}
\end{subfigure}
\caption{$k$-NN estimator (\ref{estimator}) with two noised regression functions: smooth $f_1^*(x) = 1.5 \cdot \left[ \lVert x - 0.5 \rVert / \sqrt{3} - 0.5 \right]$ for panel (a) and "sinus" $f_2^*(x) = 1.5 \cdot \textnormal{sin}(\lVert x \rVert / \sqrt{3})$ for panel (b), with uniform covariates $x_j \overset{\textnormal{i.i.d.}}{\sim} \mathbb{U}[0, 1]^3$. Each curve corresponds to the $L_2(\mathbb{P}_n)$ squared norm error for the stopping rules (\ref{k_tau}), (\ref{k_star}), (\ref{k_ho}), (\ref{k_1_out}), averaged over $1000$ independent trials, versus the sample size $n = \{50, 80, 100, 160, 200, 250 \}$.}
 \label{fig:comp_knn_art}
\end{figure}
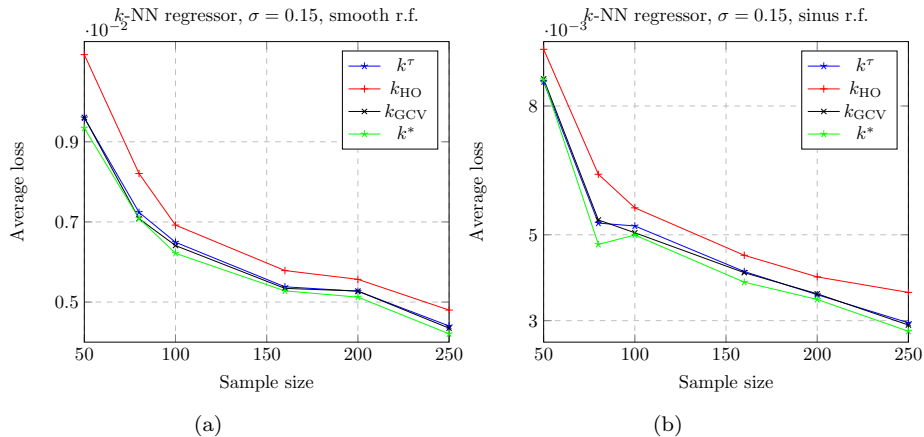

\textbf{Results of the simulation experiments.}

Figure \ref{fig:comp_knn_art} displays the resulting (averaged over $1000$ repetitions) $L_2(\mathbb{P}_n)$-error of $k^{\tau}$ (\ref{k_tau}), $k^*$ (\ref{k_star}), $k_{\text{HO}}$ (\ref{k_ho}), and $k_{\text{GCV}}$ (\ref{k_1_out}) versus the sample size $n$. In particular, Figure \ref{smooth} shows the results for the "smooth" regression function, whereas Figure \ref{sinus} provides the results for the "sinus" regression function. 

First, from all the graphs, all curves do not increase as the sample size $n$ grows. One can notice that in both graphs, the prediction error of the Hold-out strategy is the worst among the model selection criteria.

In more detail, Figure \ref{smooth} indicates that the best performance is achieved by $k^*$ (non-computable in practice bias-variance trade-off).
Besides that, the minimum discrepancy principle rule $k^{\tau}$ is uniformly better than $k_{\textnormal{HO}}$ and has the same performance as the one of $k_{\textnormal{GCV}}$. Moreover, the gap between $k^{\tau}$ and $k^*$ is getting smaller for the sample sizes $n \geq 160$. This behavior supports the theoretical part of the present paper because $k^{\tau}$ should serve as an estimator of $k^*$. Since $k^*$ is the well-known bias-variance trade-off, the minimum discrepancy principle stopping rule seems to be a meaningful model selection method.

In Figure \ref{sinus}, the best performance is achieved again by $k^*$ -- a non-computable in practice stopping rule. As for the data-driven model selection methods, the stopping rules $k^{\tau}$ and $k_{\text{GCV}}$ (an asymptotically optimal model selection strategy) perform almost equivalently.

\subsection{Real data} \label{real_data}

Here, we tested the performance (prediction error and runtime) of the early stopping rule $k^{\tau}$ (\ref{k_tau}) for choosing the hyperparameter in the $k$-NN estimator on four different data sets mostly taken from the UCI repository (\cite{Dua:2019}).

\textbf{Data sets description}


The housing data set "Boston Housing Prices" concerns the task of predicting housing values in areas of Boston (USA), the input points are $13$-dimensional. 

The "Diabetes" data set consists of 10 columns that measure different patient's characteristics (age, sex, body mass index, etc), the output is a quantitative measure of disease progression one year after the baseline. 

The "Power Plants" data set contains 9568 data points collected from a Combined Cycle Power Plant over 6 years (2006-2011), when the plant was set to work with the full load. 

"California Houses Prices" data set (\cite{pace1997sparse}) contains information from the 1990's California census. The input variables are "total bedrooms", "total rooms", etc. The output variable is the median house value for households within a block (measured in US Dollars). 

Notice that for "California Houses Prices" and "Power Plants" data sets we take the first 3000 data points in order to speed up the calculations. 

\vspace{0.2cm}

\textbf{Description of the simulation design}

Assume that we are given one of the data sets described above. Let us rescale each variable of this data set $\widetilde{x} \in \mathbb{R}^n$ such that all the components $\widetilde{x}_i, \ i = 1, \ldots, n$, belong to $[0, 1]$:
\begin{equation*}
    \widetilde{x}_i = \frac{\widetilde{x}_i - \min(\widetilde{x})}{\max(\widetilde{x}) - \min(\widetilde{x})}, \ i = 1, \ldots, n,
\end{equation*}
where $\min(\widetilde{x})$ and $\max(\widetilde{x})$ denote the minimum and the maximum component of the vector $\widetilde{x}$.

Following that, we split the data set into two parts: one is denoted $S_{\textnormal{train}} = \{ x_{\textnormal{train}}, y_{\textnormal{train}} \}$ (70 \% of the whole data) and is made for training and model selection (model selection rules $k^{\tau}$, $k_{\text{GCV}}$, $k_{\text{5FCV}}$, and $k_{\text{AIC}}$), the other one (30 \% of the whole data) is denoted $S_{\textnormal{test}} = \{ x_{\textnormal{test}}, y_{\textnormal{test}} \}$ and is made for making prediction on it. We denote $n_{\textnormal{train}}$ and $n_{\textnormal{test}}$ as the sample sizes of $S_{\textnormal{train}}$ and $S_{\textnormal{test}}$, respectively. Then, our experiments' design is divided into four parts.

\vspace{0.15cm}

At the beginning, we create a grid of sub-sample size for each data set:
\begin{equation} \label{subsamples}
    n_s \in \Big\{\floor*{n_{\textnormal{train}}/5}, \floor*{n_{\textnormal{train}}/4}, \floor*{n_{\textnormal{train}}/3}, \\ \floor*{n_{\textnormal{train}}/2}, n_{\textnormal{train}}\Big\},
\end{equation}
and a grid of the maximum number of neighbors $k_{\textnormal{max}} = 3 \floor*{\log(n_s)}$, where $n_{\textnormal{train}} = \ceil*{0.7n}$ and $n$ is the sample size of the whole data. 

Further, for each data set and sub-sample size from Eq. (\ref{subsamples}), we estimate the noise variance $\sigma^2$ from the regression model (\ref{nonparametric_regression_model}) for the AIC criterion. In our simulated experiments, we take the estimator from \cite[Eq. (5.86)]{wasserman2006all}, which is a consistent estimator of $\sigma^2$ under an assumption that $f^*$ is "sufficiently smooth".
\begin{equation} \label{var_est}
    \widehat{\sigma}^2 \coloneqq \frac{\lVert (I_{n_s} - A_k)y_{s} \rVert^2}{n_s (1 - 1/k)} \ \ \ \textnormal{ with } \ \ \ k = 2, 
\end{equation}
where $y_s$ corresponds to the vector of responses from the chosen sub-samples.
After that, we compute our stopping rule $k^{\tau}$ and other model selection strategies from Section \ref{description_early_stopping_rules}. To do that, for each data set and each integer $n_s$ from Eq. (\ref{subsamples}), we randomly sample $n_s$ data points from $S_{\textnormal{train}}$, compute the $k$-NN estimator (\ref{estimator}) and the empirical risk (\ref{empirical_risk}) for $k_{\textnormal{max}}$, and at each step of the iteration process we reduce the value of $k$ by one. Notice that one does not have to calculate the neighborhood matrix $A_k$ for each $k \in \{1, \ldots, k_{\textnormal{max}}\}$, since it is sufficient to do only for $k_{\textnormal{max}}$ (cf. Eq. (\ref{iteration_for_matrix})). This process is repeated until the empirical risk crosses the threshold $2R_2$. Fig. \ref{fig:process_learning} provides two illustrations of the minimum discrepancy strategy $k^{\tau}$ for two data sets: "Diabetes" and "Boston Housing Prices".

Following that, the AIC criterion (\ref{k_aic}), $5$--fold cross-validation (\ref{k_vfcv}), and the generalized cross-validation $k_{\textnormal{GCV}}$ are calculated: we start by defining a grid of values for $k: \{ 1, 2, \ldots, k_{\textnormal{max}} \}$, and one should compute $k_{\textnormal{AIC}}$, $k_{\textnormal{GCV}}$, and $k_{\textnormal{5FCV}}$ from Eq. (\ref{k_aic}), Eq. (\ref{k_1_out}), and Eq. (\ref{k_vfcv}) over the mentioned grid.

In the final part, given $k^{\tau}$, $k_{\text{AIC}}$, $k_{\text{5FCV}}$, and $k_{\text{GCV}}$, the goal is to make a prediction on the test data set $x_{\textnormal{test}}$. This can be done as follows. Assume that $x_0 \in x_{\textnormal{test}}$, then the prediction of the $k$-NN estimator on this point can be defined as
\begin{equation*}
    f^k(x_0) = a_k(x_0)^\top y_{s},
\end{equation*}
where $a_k(x_0) = [a_k(x_0, x_1), \ldots, a_k(x_0, x_{n_s})]^\top$ and $x_{s} = [x_1^{\top}, \ldots, x_{n_{s}}^\top]^{\top}$, with $a_k(x_0, x_i) = 1/k$ if $x_i, \ i \in \{1, \ldots, n_{\textnormal{train}} \}$, is a nearest neighbor of $x_0$, otherwise $0$. Further, one can choose $k$ to be equal $k^{\tau}$, $k_{\text{AIC}}$, $k_{\text{5FCV}}$ or $k_{\text{GCV}}$ that are already computed. Combining all the steps together, one is able to calculate the least-squares prediction error $\lVert f^k - y_{\textnormal{test}} \rVert$.

For each sub-sample size $n_s$ from Eq. (\ref{subsamples}) and data set, the procedure has to be performed $25$ times (via new sub-samples from the data set). 

\vspace{0.15cm}

\textbf{Results of the simulation experiments.}

Figures \ref{small_datasets_results} and \ref{large_datasets_results} display the averaged (over 25 repetitions) runtime (in seconds) and the prediction error of the model selection rules $k^{\tau}$ (\ref{k_tau}), $k_{\textnormal{AIC}}$ (\ref{k_aic}), 5-fold cross-validation (\ref{k_vfcv}), and generalized cross-validation (\ref{k_1_out}) for "Boston Housing Prices", "Diabetes" (in Figure \ref{small_datasets_results}), and "California Houses Prices", "Power Plants" data sets (in Figure \ref{large_datasets_results}).

Figures \ref{runtime_boston}, \ref{runtime_diabetes} indicate that the minimum discrepancy principle rule $k^{\tau}$ has the smallest runtime among the model selection criteria. At the same time, Figure \ref{prediction_error_boston} shows that the prediction error of $k^{\tau}$ has better performance than that of $k_{\textnormal{AIC}}$, $k_{\textnormal{5FCV}}$, and $k_{\textnormal{GCV}}$. Figure \ref{pred_error_diabetes} indicates that the performance of the minimum discrepancy stopping rule $k^{\tau}$ is better than that of $k_{\textnormal{AIC}}$, $k_{\textnormal{5FCV}}$, and similar to that of $k_{\textnormal{GCV}}$.
Let us turn to the results for the "California Houses Prices" and "Power Plants" data sets. Figures \ref{runtime_california}, \ref{runtime_power} show the runtime of the stopping rules: one can conclude that the computational time of the minimum discrepancy rule $k^{\tau}$ is less than the computational time of the generalized cross-validation, AIC criterion, and $5$ fold cross-validation. Figures \ref{prediction_error_california}, \ref{prediction_error_power} display the prediction performance of the model selection rules: for the "California Houses Prices" data set, the prediction performance of $k^{\tau}$ is comparable to the performance of $k_{\textnormal{GCV}}$ and is uniformly better than the $5$FCV rule $k_{\textnormal{5FCV}}$ and $k_{\textnormal{AIC}}$; for the "Power Plants" data set, the prediction error of the minimum discrepancy principle is similar to that of $k_{\textnormal{AIC}}, k_{\textnormal{GCV}}$, and $k_{\textnormal{5FCV}}$ for the sub-sample sizes $n_s \leq 1000$, and is little better than generalzed cross-validation and 5FCV for $n_s = 2100$. 

The overall conclusion from the simulation experiments is that the prediction error of the MDP stopping rule $k^{\tau}$ is often better than for standard model selection strategies, such as the AIC criterion or $5$--fold cross-validation, while its computational time is lower.   

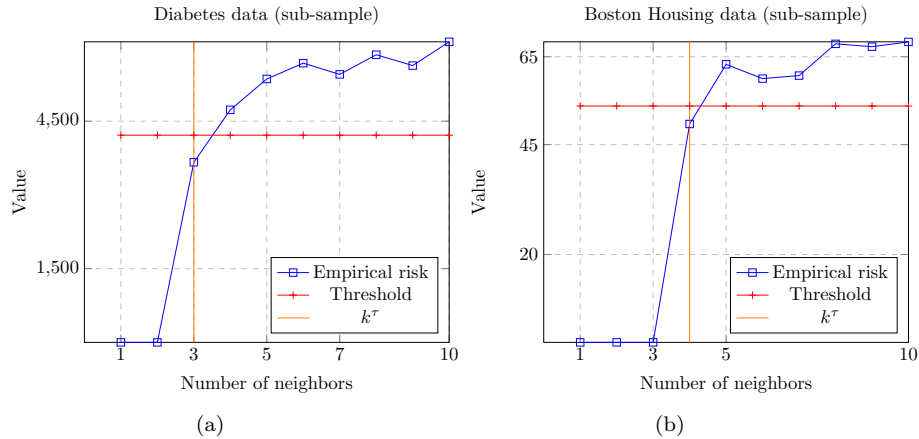
\begin{figure}
\begin{subfigure}[b]{0.45\textwidth}
\begin{tikzpicture}[scale=0.7]
\tikzstyle{every node}=[font=\normalsize]
\begin{axis}[
    title={Diabetes data (sub-sample)},
    xlabel={Number of neighbors},
    ylabel={Value},
    xmin=0, xmax=10,
    ymin=0, ymax=6114.66446281,
    xtick={1,3,5,7, 10},
    ytick={1500, 4500},
    legend pos=south east,
    ymajorgrids=true,
    xmajorgrids=true,
    grid style=dashed,
]

\addplot[
    color=blue,
    mark=square,
    ]
    coordinates {
    (1,0)(2,0)(3,3664.18333333)(4,4731.1125)(5,5358.508)(6,5680.28472222)(7,5452.15408163)(8,5851.84765625)(9,5631.3308642)(10,6114.66446281)
    };
    \addlegendentry{Empirical risk}
\addplot[
    color=red,
    mark=+,
    ]
    coordinates {
    (1,4215.1)(2,4215.1)(3,4215.1)(4,4215.1)(5,4215.1)(6,4215.1)(7,4215.1)(8,4215.1)(9,4215.1)(10,4215.1)
    };
    \addlegendentry{Threshold}
\addplot +[mark=none, color=orange] coordinates {(3, 0)(3, 6114.66446281)};
\addlegendentry{$k^{\tau}$}
\end{axis}
\end{tikzpicture}
\caption{}
\label{a}
\end{subfigure}
\qquad
\begin{subfigure}[b]{0.45\textwidth}
\begin{tikzpicture}[scale=0.7]
\tikzstyle{every node}=[font=\normalsize]
\begin{axis}[
    title={Boston Housing data (sub-sample)},
    xlabel={Number of neighbors},
    ylabel={Value},
    xmin=0, xmax=10,
    ymin=0, ymax=68.4,
    xtick={1,3,5,10},
    ytick={20, 45, 65},
    legend pos=south east,
    ymajorgrids=true,
    xmajorgrids=true,
    grid style=dashed,
]

\addplot[
    color=blue,
    mark=square,
    ]
    coordinates {
    (1,0)(2,0)(3,0)(4,49.69747283)(5,63.26987826)(6,60.0285628)(7,60.71023957)(8,67.93981658)(9,67.28213634)(10,68.37776957)
    };
    \addlegendentry{Empirical risk}
\addplot[
    color=red,
    mark=+,
    ]
    coordinates {
    (1, 53.79369565)(2, 53.79369565)(3, 53.79369565)(4, 53.79369565)(5, 53.79369565)(6,53.79369565)(7,53.79369565)(8,53.79369565)(9,53.79369565)(10,53.79369565)
    };
    \addlegendentry{Threshold}
\addplot +[mark=none, color=orange] coordinates {(4, 0)(4, 6114.66446281)};
\addlegendentry{$k^{\tau}$}
\end{axis}
\end{tikzpicture}
\caption{}
\label{b}
\end{subfigure}
\caption{Stopping the learning process based on the rule (\ref{k_tau}) applied to two data sets: a) "Diabetes" and b) "Boston Housing Prices". "Threshold" horizontal line corresponds to the estimated variance equal to $2R_2$.}
\label{fig:process_learning}
\end{figure}

\begin{figure}
\begin{subfigure}[b]{0.45\textwidth}
\begin{tikzpicture}[scale=0.65]
\tikzstyle{every node}=[font=\normalsize]
\begin{axis}[
    title={Boston H.P., runtime},
    xlabel={Sub-sample size},
    ylabel={Average time (sec)},
    xmin=55, xmax=370,
    ymin=0.0002, ymax=0.02,
    xtick={70,100,140,200,300,400},
    ytick={0.01,0.1},
    legend pos=north west,
    ymajorgrids=true,
    xmajorgrids=true,
    grid style=dashed,
]

\addplot+[
    color=green,
    mark=square,
    ][error bars/.cd,y dir=both, y explicit]
    coordinates {
    (70,0.00088333)+-(0, 2.62786401e-04)(88,0.00105345)+-(0, 2.07151095e-05)(118,0.00159296)+-(0, 2.44510842e-05)(177,0.00368924)+-(0, 5.98135386e-04)(354,0.01041318)+-(0, 1.57326189e-04)
    };
    \addlegendentry{$k^{\tau}$}
\addplot+[
    color=red,
    mark=+,
    ][error bars/.cd,y dir=both, y explicit]
    coordinates {
    (70, 0.00090446)+-(0, 2.04332719e-04)
    (88, 0.00108907)+-(0, 8.86841428e-06)
    (118, 0.001619)+-(0, 1.10579102e-05)
    (177, 0.00349167)+-(0, 8.03711066e-05)
    (354, 0.01184145)+-(0, 1.07254010e-03)
    };
    \addlegendentry{$k_{\textnormal{AIC}}$}
\addplot+[
    color=black,
    mark=x,
    ][error bars/.cd,y dir=both, y explicit]
    coordinates {
    (70, 0.00090341)+-(0, 2.56110337e-04)
    (88, 0.00108125)+-(0, 1.01114036e-05)
    (118, 0.00161082)+-(0, 1.21529918e-05)
    (177, 0.00354301)+-(0, 1.6871201e-04)
    (354, 0.01137727)+-(0, 1.32377524e-03)
    };
    \addlegendentry{$k_{\textnormal{GCV}}$}
\addplot+[
    color=blue,
    mark=*,
    ][error bars/.cd,y dir=both, y explicit]
    coordinates {
    (70, 0.00183)+-(0, 3.66798504e-04)
    (88, 0.00211412)+-(0, 1.90538661e-05)
    (118, 0.00295113)+-(0, 2.21499888e-05)
    (177, 0.00586444)+-(0, 5.12052634e-05)
    (354, 0.01828056)+-(0, 1.29027136e-03)
    };
    \addlegendentry{$k_{\textnormal{5FCV}}$}

\end{axis}
\end{tikzpicture}
\caption{}
\label{runtime_boston}
\end{subfigure}
\qquad
\begin{subfigure}[b]{0.45\textwidth}
\begin{tikzpicture}[scale=0.65]
\tikzstyle{every node}=[font=\normalsize]
\begin{axis}[
    title={Boston H.P., prediction error},
    xlabel={Sub-sample size},
    ylabel={Average loss},
    xmin=55, xmax=370,
    ymin=50, ymax=85,
    xtick={70,100,140,200,300,400},
    ytick={60, 70, 80},
    legend pos=north east,
    ymajorgrids=true,
    xmajorgrids=true,
    grid style=dashed,
]

\addplot+[
    color=green,
    mark=square,
    ][error bars/.cd,y dir=both, y explicit]
    coordinates {
    (70,75.20453421)+-(0, 6.21484738e+00)(88,73.5148997)+-(0, 3.07344191e+00)(118,70.06091264)+-(0, 4.37484465e+00)(177,62.65990106)+-(0, 4.68084508e+00)(354,51.64092369)+-(0, 7.10542736e-15)
    };
    \addlegendentry{$k^{\tau}$}
\addplot+[
    color=red,
    mark=+,
    ][error bars/.cd,y dir=both, y explicit]
    coordinates {
    (70, 75.37342967)+-(0, 5.81410471)(88, 75.07509489)+-(0, 4.22851383)(118, 70.5055555)+-(0, 4.77139767)(177, 62.95261192)+-(0, 5.4704744)(354, 52.57261645)+-(0, 0)
    };
    \addlegendentry{$k_{\textnormal{AIC}}$}
\addplot+[
    color=black,
    mark=x,
    ][error bars/.cd,y dir=both, y explicit]
    coordinates {
    (70, 75.13191115)+-(0, 5.65246733e+00)(88, 72.99322637)+-(0, 3.56012765e+00)(118, 70.84480346)+-(0, 4.38034112e+00)(177, 63.34763725)+-(0, 5.51154089e+00)(354, 53.55843123)+-(0, 0)
    };
    \addlegendentry{$k_{\textnormal{GCV}}$}
\addplot+[
    color=blue,
    mark=*,
    ][error bars/.cd,y dir=both, y explicit]
    coordinates {
    (70, 75.19499164)+-(0, 7.11206602)(88, 73.52951529)+-(0, 6.68918027)(118, 70.9866001)+-(0, 4.71953746)(177, 64.23005828)+-(0, 4.11275903)(354, 52.27447477)+-(0, 0.43461245)
    };
    \addlegendentry{$k_{\textnormal{5FCV}}$}
\end{axis}
\end{tikzpicture}
\caption{}
\label{prediction_error_boston}
\end{subfigure}
\\
\begin{subfigure}[b]{0.45\textwidth}
\begin{tikzpicture}[scale=0.65]
\tikzstyle{every node}=[font=\normalsize]
\begin{axis}[
    title={Diabetes, runtime},
    xlabel={Sub-sample size},
    ylabel={Average time (sec)},
    xmin=50, xmax=320,
    ymin=0.0001, ymax=0.01,
    xtick={61,100,140,200,300},
    ytick={0.01,0.05},
    legend pos=north west,
    ymajorgrids=true,
    xmajorgrids=true,
    grid style=dashed,
]

\addplot+[
    color=green,
    mark=square,
    ][error bars/.cd,y dir=both, y explicit]
    coordinates {
    (61,0.00066955)+-(0, 2.41996690e-04)(77,0.00086998)+-(0, 3.36777956e-05)(103,0.00119806)+-(0, 2.87908164e-05)(154,0.0026118)+-(0, 5.72568209e-04)(309,0.00823844)+-(0, 6.51476927e-04)
    };
    \addlegendentry{$k^{\tau}$}
\addplot+[
    color=red,
    mark=+,
    ][error bars/.cd,y dir=both, y explicit]
    coordinates {
    (61, 0.00075775)+-(0, 1.38725235e-04)
    (77, 0.00091011)+-(0, 1.09472341e-05)
    (103, 0.00131685)+-(0, 1.51961531e-05)
    (154, 0.00291117)+-(0, 2.84405746e-04)
    (309, 0.00897405)+-(0, 6.02445653e-04)
    };
    \addlegendentry{$k_{\textnormal{AIC}}$}
\addplot+[
    color=black,
    mark=x,
    ][error bars/.cd,y dir=both, y explicit]
    coordinates {
    (61, 0.00075202)+-(0, 1.54020749e-04)
    (77, 0.0009038)+-(0, 1.58490802e-05)
    (103, 0.00130192)+-(0, 1.21148323e-05)
    (154, 0.00288856)+-(0, 3.39537075e-04)
    (309, 0.00875706)+-(0, 4.37810795e-04)
    };
    \addlegendentry{$k_{\textnormal{GCV}}$}
\addplot+[
    color=blue,
    mark=*,
    ][error bars/.cd,y dir=both, y explicit]
    coordinates {
    (61, 0.00157907)+-(0, 2.88243091e-04)
    (77, 0.00179781)+-(0, 1.48375420e-05)
    (103, 0.00243689)+-(0, 2.19849835e-05)
    (154, 0.00477407)+-(0, 2.71445501e-04)
    (309, 0.01394191)+-(0, 5.62921184e-04)
    };
    \addlegendentry{$k_{\textnormal{5FCV}}$}
\end{axis}
\end{tikzpicture}
\caption{}
\label{runtime_diabetes}
\end{subfigure}
\qquad
\begin{subfigure}[b]{0.45\textwidth}
\begin{tikzpicture}[scale=0.65]
\tikzstyle{every node}=[font=\normalsize]
\begin{axis}[
    title={Diabetes, prediction error},
    xlabel={Sub-sample size},
    ylabel={Average loss},
    xmin=50, xmax=320,
    ymin=600, ymax=770,
    xtick={61,100,140,200,300},
    ytick={650, 700, 725},
    legend pos=north east,
    ymajorgrids=true,
    xmajorgrids=true,
    grid style=dashed,
]

\addplot+[
    color=green,
    mark=square,
    ][error bars/.cd,y dir=both, y explicit]
    coordinates {
    (61,694.73549068)+-(0, 35.52865567)(77,676.43197929)+-(0, 26.84936344)(103,664.67351685)+-(0, 32.89202372)(154,641.60369852)+-(0, 21.91584539)(309,629.4821625)+-(0, 0)
    };
    \addlegendentry{$k^{\tau}$}
\addplot+[
    color=red,
    mark=+,
    ][error bars/.cd,y dir=both, y explicit]
    coordinates {
    (61, 705.34847532)+-(0, 49.93668494)
    (77, 674.62784833)+-(0, 33.60500198)
    (103, 666.22093732)+-(0, 33.09974021)
    (154, 649.06815285)+-(0, 27.21999454)
    (309, 629.4821625)+-(0, 0)
    };
    \addlegendentry{$k_{\textnormal{AIC}}$}
\addplot+[
    color=black,
    mark=x,
    ][error bars/.cd,y dir=both, y explicit]
    coordinates {
    (61, 698.73549068)+-(0, 40.52865567)
    (77, 676.43197929)+-(0, 33.84936344)
    (103, 665.67351685)+-(0, 32.89202372)
    (154, 644.60369852)+-(0, 21.91584539)
    (309, 629.4821625)+-(0, 0)
    };
    \addlegendentry{$k_{\textnormal{GCV}}$}
\addplot+[
    color=blue,
    mark=*,
    ][error bars/.cd,y dir=both, y explicit]
    coordinates {
    (61, 725.96038738)+-(0, 68.27957763)
    (77, 688.52927512)+-(0, 41.16554674)
    (103, 667.82796817)+-(0, 27.48797972)
    (154, 652.75239488)+-(0, 27.13117145)
    (309, 633.99680041)+-(0, 5.80998993)
    };
    \addlegendentry{$k_{\textnormal{5FCV}}$}
\end{axis}
\end{tikzpicture}
\caption{}
\label{pred_error_diabetes}
\end{subfigure}
\caption{Runtime (in seconds) and $L_2(\mathbb{P}_n)$ prediction error versus sub-sample size for different model selection methods: MD principle (\ref{k_tau}), AIC (\ref{k_aic}), GCV (\ref{k_1_out}), and $5$--fold cross-validation (\ref{k_vfcv}), tested on the "Boston Housing Prices" and "Diabetes" data sets. In all cases, each point corresponds to the average of $25$ trials. (a), (c) Runtime verus the sub-sample size $n \in \{70, 88, 118, 177, 354 \}$. (b), (d) Least-squares prediction error $\lVert f^k - y_{\textnormal{test}} \rVert$ versus the sub-sample size $n \in \{70, 88, 118, 177, 354 \}$.}
\label{small_datasets_results}
\end{figure}
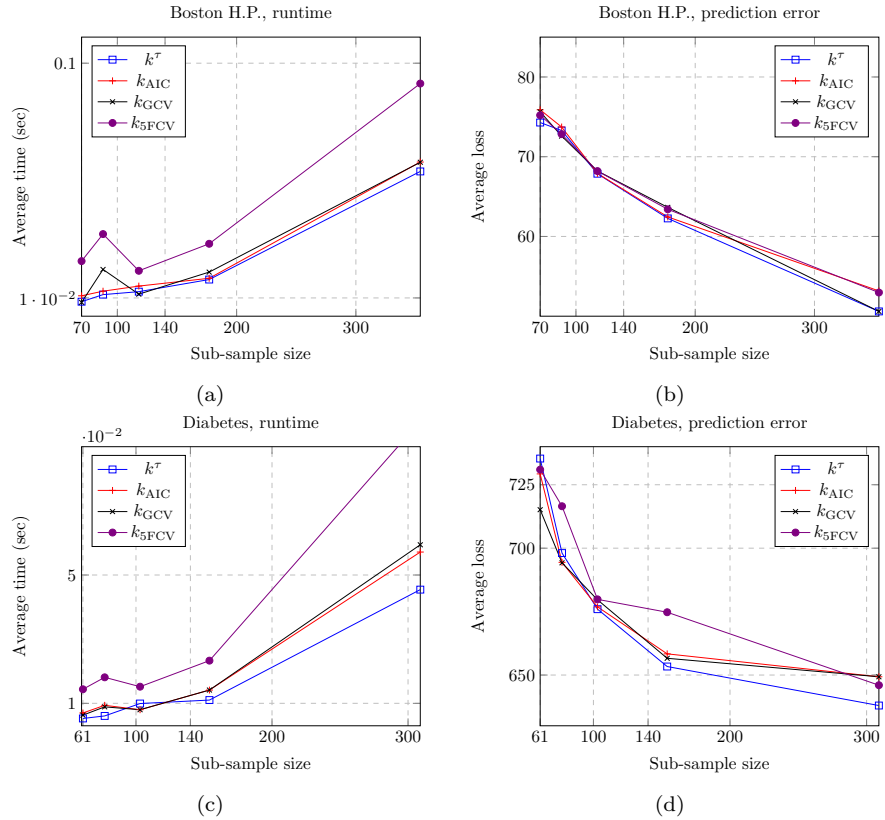

\begin{figure}
\begin{subfigure}[b]{0.45\textwidth}
\begin{tikzpicture}[scale=0.65]
\tikzstyle{every node}=[font=\normalsize]
\begin{axis}[
    title={California H.P., runtime},
    xlabel={Sub-sample size},
    ylabel={Average time (sec)},
    xmin=400, xmax=2120,
    ymin=0.001, ymax=0.7,
    xtick={500,700,1000,2000},
    ytick={0.1,0.5, 1},
    legend pos=north west,
    ymajorgrids=true,
    xmajorgrids=true,
    grid style=dashed,
]

\addplot+[
    color=green,
    mark=square,
    ][error bars/.cd,y dir=both, y explicit]
    coordinates {
    (420,0.0150546)+-(0, 0.0009536)(525,0.02358903)+-(0, 0.00385092)(700,0.04246296)+-(0, 0.00421129)(1050,0.09039716)+-(0, 0.00689664)(2100,0.34701059)+-(0, 0.02989554)
    };
    \addlegendentry{$k^{\tau}$}
\addplot+[
    color=red,
    mark=+,
    ][error bars/.cd,y dir=both, y explicit]
    coordinates {
    (420,0.01531987)+-(0, 0.00086377)(525,0.02489831)+-(0, 0.0070419)(700,0.04324284)+-(0, 0.00471381)(1050,0.09195917)+-(0, 0.00139618)(2100,0.37668136)+-(0, 0.02133098)
    };
    \addlegendentry{$k_{\textnormal{AIC}}$}
\addplot+[
    color=black,
    mark=x,
    ][error bars/.cd,y dir=both, y explicit]
    coordinates {
    (420,0.01521532)+-(0, 0.00063996)(525,0.02342019)+-(0, 0.00098027)(700,0.04189021)+-(0, 0.00140984)(1050,0.09239346)+-(0, 0.00246034)(2100,0.3743844)+-(0, 0.02224616)
    };
    \addlegendentry{$k_{\textnormal{GCV}}$}
\addplot+[
    color=blue,
    mark=*,
    ][error bars/.cd,y dir=both, y explicit]
    coordinates {
    (420,0.0247823)+-(0, 0.00141235)(525,0.03641508)+-(0, 0.00178887)(700,0.06539243)+-(0, 0.00227723)(1050,0.13690432)+-(0, 0.00297915)(2100,0.57466645)+-(0, 0.03754322)
    };
    \addlegendentry{$k_{\textnormal{5FCV}}$}
\end{axis}
\end{tikzpicture}
\caption{}
\label{runtime_california}
\end{subfigure}
\qquad
\begin{subfigure}[b]{0.45\textwidth}
\begin{tikzpicture}[scale=0.65]
\tikzstyle{every node}=[font=\normalsize]
\begin{axis}[
    title={California H.P., prediction error},
    xlabel={Sub-sample size},
    ylabel={Average loss},
    xmin=400, xmax=2125,
    ymin=11, ymax=14.3,
    xtick={500,700,1000,2000},
    ytick={12 ,13 , 13.5},
    legend pos=north east,
    ymajorgrids=true,
    xmajorgrids=true,
    grid style=dashed,
]

\addplot+[
    color=green,
    mark=square,
    ][error bars/.cd,y dir=both, y explicit]
    coordinates {
    (420,13.26266873)+-(0, 4.08428839e-01)(525,13.04593567)+-(0, 4.52360769e-01)(700,12.65067224)+-(0, 2.93333571e-01)(1050,12.18972854)+-(0, 1.70809686e-01)(2100,11.61902792)+-(0, 1.77635684e-15)
    };
    \addlegendentry{$k^{\tau}$}
\addplot+[
    color=red,
    mark=+,
    ][error bars/.cd,y dir=both, y explicit]
    coordinates {
    (420,13.3280656)+-(0, 4.79305005e-01)(525,13.04758682)+-(0,4.84281851e-01)(700,12.73186305)+-(0, 4.09857402e-01)(1050,12.25079929)+-(0,2.40823949e-01)(2100,11.61902792)+-(0,1.77635684e-15)
    };
    \addlegendentry{$k_{\textnormal{AIC}}$}
\addplot+[
    color=black,
    mark=x,
    ][error bars/.cd,y dir=both, y explicit]
    coordinates {
    (420,13.17929351)+-(0,3.52157250e-01)(525,12.98505496)+-(0,3.75196676e-01)(700,12.56349774)+-(0,1.97871822e-01)(1050,12.17333137)+-(0,1.98109715e-01)(2100,11.5611128)+-(0,1.77635684e-15)
    };
    \addlegendentry{$k_{\textnormal{GCV}}$}
\addplot+[
    color=blue,
    mark=*,
    ][error bars/.cd,y dir=both, y explicit]
    coordinates {
    (420,13.32115187)+-(0,0.43444179)(525,13.05816783)+-(0,0.39781389)(700,12.61299154)+-(0,0.29451386)(1050,12.247316)+-(0,0.26074582)(2100,11.64718134)+-(0,0.11529009)
    };
    \addlegendentry{$k_{\textnormal{5FCV}}$}
\end{axis}
\end{tikzpicture}
\caption{}
\label{prediction_error_california}
\end{subfigure}
\\
\begin{subfigure}[b]{0.45\textwidth}
\begin{tikzpicture}[scale=0.65]
\tikzstyle{every node}=[font=\normalsize]
\begin{axis}[
    title={Power Plants, runtime},
    xlabel={Sub-sample size},
    ylabel={Average time (sec)},
    xmin=350, xmax=2125,
    ymin=0.01, ymax=0.4,
    xtick={420,700,1000,2000},
    ytick={0.1, 1, 2},
    legend pos=north west,
    ymajorgrids=true,
    xmajorgrids=true,
    grid style=dashed,
]

\addplot+[
    color=green,
    mark=square,
    ][error bars/.cd,y dir=both, y explicit]
    coordinates {
    (420,0.01470181)+-(0,0.00056176)(525,0.02205186)+-(0,0.00056347)(700,0.03817163)+-(0,0.0006976)(1050,0.08331034)+-(0,0.00220411)(2100,0.34804553)+-(0,0.00859642)
    };
    \addlegendentry{$k^{\tau}$}
\addplot+[
    color=red,
    mark=+,
    ][error bars/.cd,y dir=both, y explicit]
    coordinates {
    (420, 0.015403)+-(0,0.00097315)(525, 0.02319606)+-(0,0.00120454)(700, 0.03987282)+-(0,0.0009613)(1050, 0.08946141)+-(0,0.0018596)(2100, 0.36750357)+-(0,0.00202153)
    };
    \addlegendentry{$k_{\textnormal{AIC}}$}
\addplot+[
    color=black,
    mark=x,
    ][error bars/.cd,y dir=both, y explicit]
    coordinates {
    (420, 0.0153015)+-(0,0.00062684)(525, 0.02296844)+-(0,0.00107175)(700, 0.03968534)+-(0,0.00119149)(1050, 0.08918223)+-(0,0.00162629)(2100, 0.36579443)+-(0,0.00204679)
    };
    \addlegendentry{$k_{\textnormal{GCV}}$}
\addplot+[
    color=blue,
    mark=*,
    ][error bars/.cd,y dir=both, y explicit]
    coordinates {
    (420, 0.02395827)+-(0,0.00058146)(525, 0.03542266)+-(0,0.00115087)(700, 0.06134853)+-(0,0.00216597)(1050, 0.13237996)+-(0,0.00174208)(2100, 0.54843725)+-(0,0.00507977)
    };
    \addlegendentry{$k_{\textnormal{5FCV}}$}
\end{axis}
\end{tikzpicture}
\caption{}
\label{runtime_power}
\end{subfigure}
\qquad
\begin{subfigure}[b]{0.45\textwidth}
\begin{tikzpicture}[scale=0.65]
\tikzstyle{every node}=[font=\normalsize]
\begin{axis}[
    title={Power Plants, prediction error},
    xlabel={Sub-sample size},
    ylabel={Average loss},
    xmin=350, xmax=2125,
    ymin=115, ymax=140,
    xtick={420,700,1000,2000},
    ytick={120, 125, 130},
    legend pos=north east,
    ymajorgrids=true,
    xmajorgrids=true,
    grid style=dashed,
]

\addplot+[
    color=green,
    mark=square,
    ][error bars/.cd,y dir=both, y explicit]
    coordinates {
    (420,135.38159152)+-(0,3.18628667)(525,132.82363088)+-(0,3.20593179)(700,129.14202151)+-(0,2.40451842)(1050,124.61522052)+-(0,1.76421490)(2100,115.75029431)+-(0,3.12638804e-14)
    };
    \addlegendentry{$k^{\tau}$}
\addplot+[
    color=red,
    mark=+,
    ][error bars/.cd,y dir=both, y explicit]
    coordinates {
    (420, 135.102044)+-(0,3.00087301)(525, 133.51627766)+-(0,3.94627588)(700, 128.98041813)+-(0,2.22865225)(1050, 124.45962331)+-(0,2.39771796)(2100, 115.75029431)+-(0,3.12638804e-14)
    };
    \addlegendentry{$k_{\textnormal{AIC}}$}
\addplot+[
    color=black,
    mark=x,
    ][error bars/.cd,y dir=both, y explicit]
    coordinates {
    (420, 134.21866286)+-(0,2.36391082)(525, 132.25980994)+-(0,2.86302466)(700, 128.47490889)+-(0,1.93164718)(1050, 123.96947848)+-(0,1.51875349)(2100, 117.13269495)+-(0,1.90658623e-14)
    };
    \addlegendentry{$k_{\textnormal{GCV}}$}
\addplot+[
    color=blue,
    mark=*,
    ][error bars/.cd,y dir=both, y explicit]
    coordinates {
    (420, 135.20225617)+-(0,2.79074689)(525, 132.55763574)+-(0,2.62206904)(700, 128.44957247)+-(0,1.81616222)(1050, 124.48593776)+-(0,1.5866712)(2100, 116.53285035)+-(0,0.73161339)
    };
    \addlegendentry{$k_{\textnormal{5FCV}}$}
\end{axis}
\end{tikzpicture}
\caption{}
\label{prediction_error_power}
\end{subfigure}
\caption{Runtime (in seconds) and $L_2(\mathbb{P}_n)$ prediction error versus sub-sample size for different model selection methods: MDP (\ref{k_tau}), AIC (\ref{k_aic}), GCV (\ref{k_1_out}), and $5$--fold cross-validation (\ref{k_vfcv}), tested on the "California Houses Prices" and "Power Plants" data set. In all cases, each point corresponds to the average of $25$ trials. (a), (c) Runtime verus the sub-sample size $n \in \{420, 525, 700, 1050, 2100 \}$. (b), (d) Least-squares prediction error $\lVert f^k - y_{\textnormal{test}} \rVert$ versus the sub-sample size $n \in \{420, 525, 700, 1050, 2100 \}$.}
\label{large_datasets_results}
\end{figure}
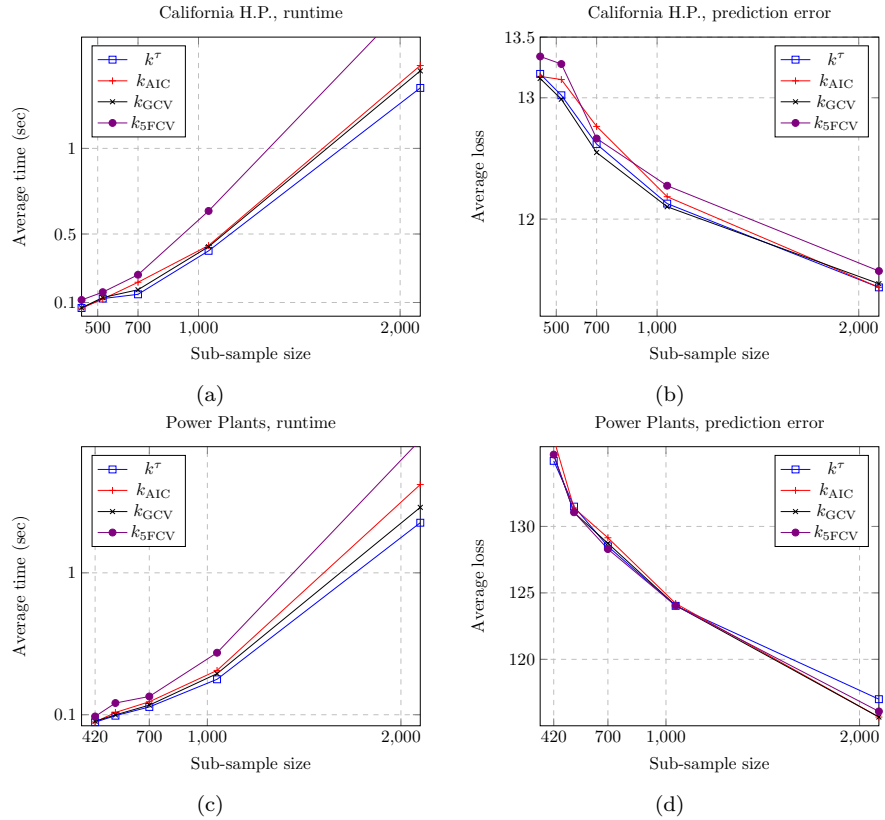

\section{Conclusion} \label{sec:conc}
In the present paper, we tackled the problem of choosing the tuning parameter $k$ in the $k$-NN regression estimator. A strategy based on early stopping and the minimum discrepancy principle was proposed. This strategy did not require any hold-out data and parameters to tune. In Section \ref{sec:optimality_result}, it was shown that the minimum discrepancy stopping rule $k^{\tau}$ (\ref{k_tau}) provides a minimax-optimal estimator, in particular, over the class of Lipschitz functions on a bounded domain. Besides that, this theoretical result was confirmed empirically on artificial and real data sets: the stopping rule has better statistical performance than other model selection rules, such as Hold-out, $5$--fold cross-validation or generalized cross-validation, while reducing the computational time of the model selection procedure. 

As for perspectives of this work, we are interested in the theoretical performance of the Nadaraya-Watson regressor (\cite{nadaraya1964estimating, wasserman2006all}). It should be close to the one of the k-NN regression estimator since these two non-parametric estimators are related (see the monographs \cite{gyorfi2006distribution, tsybakov2008introduction}). The main difficulty should come from the fact that, if $h$ is the bandwidth parameter and $A_h$ is the smoothing matrix of the Nadaraya-Watson estimator, then $\textnormal{tr}(A_h^{\top}A_h) \neq \textnormal{tr}(A_h)$. This fact implies that the expectation of the empirical risk minus the noise variance will not be equal to the difference between the bias and variance terms (see Eq. (\ref{exp_empirical_risk})). Therefore, there should be another concentration result that deals with this problem. Besides that, we should emphasize that the early stopping rules in this work were estimating the famous
bias-variance trade-off (\cite[Chapter 7]{hastie2009elements}). However, in (\cite{belkin2019reconciling, belkin2020two}) the bias-variance balancing paradigm was rethought by discovering some settings (exact fit to the data) for which a phenomenon of the "double descent" of the risk curve appeared. It would be interesting to understand if early stopping can work for these settings. The reader can look at (\cite{dwivedi2020revisiting}) and references therein for another reexamination of the paradigm.
Another future direction could be modifying the minimum discrepancy principle strategy such that it can escape the so-called curse of dimensionality (\cite{tsybakov2008introduction, hastie2009elements}) in the minimax-optimal rate. The reader can go through (\cite{zhao2021efficient}) for an example of a method to choose $k$ in the classification setting where it was possible.

\section{Availability of data and material}
The real-world data description is available in Section \ref{real_data}.

\section{Code availability}
Code for the manuscript is provided at
\href{https://github.com/YaroslavAveryanov/Minimum-discrepancy-principle-for-choosing-k}{https://github.com/YaroslavAveryanov/Minimum-discrepancy-principle-for-choosing-k}. 

\section{Competing interests}
The authors declare that they have no competing interests.

\section{Consent for publication}
Not applicable.

\section{Ethics approval and consent to participate}
Not applicable.

\section{Funding}
None.

\printbibliography

\newpage

\appendix \label{Appendix}

\begin{center}
\textbf{Supplementary material}
\end{center}

Below, one can find a plan of Appendix. 

\vspace{0.2cm}

In Appendix \ref{auxiliary_lemmas}, we state some already known results that will be used along the other sections of Appendix. 

\vspace{0.2cm}

Appendix \ref{main_quantities} is devoted to the introduction of the main quantities for the derivation of the proofs.

\vspace{0.2cm}

The main goal of Appendix \ref{var_control} is to provide a concentration inequality for the difference of the variance $V(k^{\tau})$ and its stochastic part $\lVert A_{k^{\tau}} \varepsilon \rVert_n^2$ as well as a concentration inequality for $ \mid R_k - \mathbb{E}_{\varepsilon}R_k \mid$.

\vspace{0.2cm}

In Appendix \ref{appendix_variance_deviation}, we derived a concentration inequality for controlling the variance term.

\vspace{0.2cm}

Appendix \ref{appendix_bias_deviation} is devoted to the derivation of a concentration inequality that deals with the deviation of the bias term.

\vspace{0.2cm}

Combining all the results from Appendices \ref{var_control}, \ref{appendix_variance_deviation}, and \ref{appendix_bias_deviation}, we are able to provide a proof of Theorem \ref{main_th}. 

\vspace{0.5cm}

\section{Auxiliary lemmas} \label{auxiliary_lemmas}

The first result is concerned with the derivation of the concentration of a Gaussian linear form around zero. 

\begin{lemma}[Concentration of a linear term] \label{linear_term_concentration}
   Let $\xi$ be a standard Gaussian vector with variance $\sigma^2$ in $\mathbb{R}^n$, $\alpha \in \mathbb{R}^n$ and $Z \coloneqq \langle \xi, \alpha \rangle = \sum_{j=1}^n \alpha_j \xi_j$. Then for every $t > 0$, one has
   \begin{equation*}
       \mathbb{P}_{\xi} \left( |Z| \geq t \right) \leq 2 \exp \left[ - \frac{t^2}{2 \sigma^2 \lVert \alpha \rVert^2}\right].
   \end{equation*}

\end{lemma}

Further, we need to recall a concentration result for a quadratic form of Gaussian random variables.

\begin{lemma}[Hanson-Wright's inequality for Gaussian random variables in \cite{rudelson2013hanson}] \label{quadratic_term_concentration}
If $\varepsilon = (\varepsilon_1, \ldots, \varepsilon_n) \overset{\textnormal{i.i.d.}}{\sim} \mathcal{N}(0, \sigma^2 I_n)$ and $A$ is a $n \times n$ matrix, then for any $t > 0$,
\begin{equation}
    \mathbb{P}_{\varepsilon} \left( |\varepsilon^\top A \varepsilon - \mathbb{E}_{\varepsilon}[\varepsilon^\top A \varepsilon]| \geq t \right) \leq 2 \exp \left[ -c \min \left( \frac{t^2}{\sigma^4 \lVert A \rVert_{F}^2}, \frac{t}{\sigma^2 \lVert A \rVert_{2}} \right) \right].
\end{equation}
\end{lemma}
The lemma below will help us transfer the results from the empirical $L_2(\mathbb{P}_n)$ norm to the $L_2(\mathbb{P}_X)$ norm (cf Appendix \ref{section_proof_corollary}).
\begin{lemma}[Hoeffding's inequality for bounded differences in \cite{wainwright2019high}, p.454] \label{hoeffding_concentration}
If the function $f$ is uniformly bounded, that is, if $\lVert f \rVert_{\infty} \coloneqq \underset{x \in \mathcal{X}}{\sup}\mid f(x) \mid \leq b$ for some $b < \infty$, then for every $t > 0$,
\begin{equation}
    \mid \lVert f \rVert_n^2 - \lVert f \rVert_2^2 \mid < t 
\end{equation}
with probability at least $1 - 2 \exp \left( - \frac{n t^2}{2 b^4} \right)$.
\end{lemma}
%
%
The next lemma provides us with an upper bound on the spectral norm of the matrix $I_n - A_k$.  
\begin{lemma} \label{operator_norm_lemma}
Recall that $\mathcal{N}_k(x_i)$ denotes the set of the $k$ nearest neighbors of $x_i, \ i = 1, \ldots, n$. For any $k \in \{1, \ldots, n\}$, define the matrix $M_k \in \mathbb{R}^{n \times n}$ as 
\begin{equation*}
    \left(M_k\right)_{ij}=
    \begin{cases}
    1 - 1/k, \textnormal{ if } i = j,\\
    0, \textnormal{ if } j \notin \mathcal{N}_k(x_i),\\
    -1/k, \textnormal{ if } j \in \mathcal{N}_k(x_i).
    \end{cases}    
\end{equation*}
Then $\lVert M_k \rVert_2^2 \leq 4$.
\end{lemma}

\begin{proof}

Due to the Gershgorin's circle theorem, the eigenvalues of $A_{k}$ satisfies 
\begin{equation*}
    \mid \lambda_l - \frac{1}{k} \mid \leq \frac{k-1}{k}, \ l = 1, \ldots, n,
\end{equation*}
which implies for any $l \in \{1, \ldots, n \}$, $1 - \lambda_l \leq 2$. It implies that for the eigenvalues of $(I_n - A_k)^\top (I_n - A_k)$:
\begin{equation*}
    \mu_l \leq 4, \ l = 1, \ldots n.
\end{equation*}
Therefore $\lVert I_n - A_{k} \rVert_2^2 = \max [\mu_l] \leq 4$.
\end{proof}

\begin{lemma} \label{variance_aux_lemma}
    For any $k \in \{2, \ldots, n\}$,
    \begin{equation*}
        \frac{1}{2}V(k - 1) \leq V(k) \leq V(k - 1).
    \end{equation*}
\end{lemma}

\begin{proof}
It is sufficient to notice that
\begin{equation*}
    V(k - 1) - V(k) = \frac{\sigma^2}{k (k - 1)} \leq \frac{\sigma^2}{k} = V(k).
\end{equation*}
\end{proof}

\begin{lemma}[Lower bound on risk in the empirical norm for Lipschitz functions] \label{minimax_bound_emp_norm}
    If the covariates $\{x_1, \ldots, x_n \}$ lie in the regularly spaced lattice in $[0, 1]^d$, then for any measurable of the input data $\widehat{f}$,
    \begin{equation}
        \underset{f^* \in \mathcal{F}_{\textnormal{Lip}}(L)}{\sup}\left[\mathbb{E}_{\varepsilon}\lVert \widehat{f} - f^* \rVert_n^2 \right] \geq c_l L^{\frac{2d}{d+2}}\left( \frac{\sigma^2}{n} \right)^{\frac{2}{2+d}}. 
    \end{equation}
\end{lemma}

\begin{proof}
    It is a standard result, and the proof is based on the Fano inequality (\cite{yu1997assouad, tsybakov2008introduction}) by reducing the estimation problem to multiple testing.  
\end{proof}

\section{Main quantities and notations} \label{main_quantities}

For more theoretical convenience (the variance term will be an increasing function, and the empirical risk will be approximately a decreasing function), define the following notation and stopping rules:
\begin{equation} \label{map_k_tilde_k}
   \lambda[k] \coloneqq \textnormal{tr}(A_k) = n/k \ \in \ \{ 1, n/(n-1), n/(n-2), \ldots, n \}
\end{equation}
and
\begin{align} \label{stopping_times_tilde}
    \begin{split}
    \lambda_1^* &\coloneqq \inf \left\{ \lambda \in \left\{ 1, \ldots, n  \right\} \mid B^2(\lambda) + \sigma^2 - 2\mathbb{E}_{\varepsilon}R_2 \leq V(\lambda) \right\}, \enspace \lambda_1^\tau \coloneqq \inf \left\{ \lambda \in \left\{ 1, \ldots, n  \right\} \mid R_{\lambda} \leq 2 R_2 \right\}\\
    \lambda_2^* &\coloneqq \sup \left\{ \lambda \in \left\{ 1, \ldots, n  \right\} \mid B^2(\lambda) + \sigma^2 - 2\mathbb{E}_{\varepsilon}R_2 \geq V(\lambda) \right\}, \enspace \lambda_2^\tau \coloneqq \sup \left\{ \lambda \in \left\{ 1, \ldots, n \right\} \mid R_{\lambda} \geq 2 R_2 \right\}.
    \end{split}
\end{align}

Notice that there is a one-to-one map between $k$ and $\lambda[k]$ as it is suggested in Eq. (\ref{map_k_tilde_k}). 

In Eq. (\ref{stopping_times_tilde}), we omit for simplicity the notation $\lambda[k]$. Moreover, in Eq. (\ref{stopping_times_tilde}) we used the notation $A_{\lambda[k]}$ (inside the definitions of $B^2(\lambda), V(\lambda)$, and $R_\lambda$) to denote the matrix $A_{k}$, for $k = n / \lambda$  corresponding to $\lambda$, i.e., $A_{\lambda[k]} \equiv A_k$.  

If $\lambda_2^*$ does not exists, set $\lambda_2^* = 1$. If $\lambda_2^{\tau}$ does not exist, set $\lambda_2^{\tau} = 1$. 


Note the bias, variance, and (expected) empirical risk at $\lambda_1^{\tau}$ are equal to the bias, variance, (expected) empirical risk at $k^{\tau}$ defined in Eq. (\ref{k_tau}), respectively. 


The behavior of the bias term, variance, risk error, and (expected) empirical risk w.r.t. the new notation $\lambda$ is presented in Fig. \ref{fig:bvr_tilde}. One can conclude that only the variance term is monotonic w.r.t. $\lambda$ (it is an increasing function).

\begin{figure}
\begin{subfigure}[b]{0.45\textwidth}
\begin{tikzpicture}[scale=0.6]
\begin{axis}[
    title={},
    xlabel={$\lambda$},
    ylabel={Value},
    xmin=4, xmax=50,
    ymin=0, ymax=0.016,
    xtick={4,10,25,50},
    ytick={0.001, 0.003, 0.01},
    legend pos=north east,
    ymajorgrids=true,
    xmajorgrids=true,
    grid style=dashed,
]

\addplot[
    color=blue,
    mark=square,
    ]
    coordinates {
    (4.16666667,1.48684339e-02)(4.54545455,1.35524382e-02)(5,1.39197664e-02)(5.55555556,1.32405466e-02)(6.25,1.35185734e-02)(7.14285714,1.28441987e-02)(8.33333333,1.07711870e-02)(10,9.94070408e-03)(12.5,9.45641655e-03)(16.66666667,8.45967619e-03)(25,6.61878667e-03)(50,0)
    };
    \addlegendentry{Empirical risk}
\addplot[
    color=red,
    mark=+,
    ]
    coordinates {
    (4.16666667,4.63324504e-03)(4.54545455,4.02944709e-03)(5,3.72855344e-03)(5.55555556,3.21561643e-03)(6.25,3.23737977e-03)(7.14285714,2.95873043e-03)(8.33333333,3.16401010e-03)(10,2.72880054e-03)(12.5,2.59415471e-03)(16.66666667,1.84044337e-03)(25,1.16351911e-03)(50, 0)
    };
    \addlegendentry{Bias}
\addplot[
    color=violet,
    mark=star,
    ]
    coordinates {
    (4.16666667,0.00083333)(4.54545455,0.00083333)(5,0.001)(5.55555556,0.00111111)(6.25,0.00125)(7.14285714,0.00142857)(8.33333333,0.00166667)(10,0.002)(12.5,0.0025)(16.66666667,0.00333333)(25,0.005)(50,0.01)
    };
    \addlegendentry{Variance}
\addplot[
    color=orange,
    mark=x,
    ]
    coordinates {
    (4.16666667,0.00546658)(4.54545455,0.00493854)(5,0.00472855)(5.55555556,0.00432673)(6.25,0.00448738)(7.14285714,0.0043873)(8.33333333,0.00483068)(10,0.0047288)(12.5,0.00509415)(16.66666667,0.00517378)(25,0.00616352)(50,0.01)
    };
    \addlegendentry{Risk}
\addplot[
    color=black,
    mark=*,
    ]
    coordinates {
    (4.16666667,1.37999117e-02)(4.54545455,1.31203562e-02)(5,1.27285534e-02)(5.55555556,1.21045053e-02)(6.25,1.19873798e-02)(7.14285714,1.15301590e-02)(8.33333333,1.14973434e-02)(10,1.07288005e-02)(12.5,1.00941547e-02)(16.66666667,8.50711004e-03)(25,6.16351911e-03)(50,0)
    };
    \addlegendentry{Exp. empirical risk}
\end{axis}
\end{tikzpicture}
\caption{}
\label{fig:bvr_tilde}
\end{subfigure}
\qquad
\begin{subfigure}[b]{0.45\textwidth}
\begin{tikzpicture}[scale=0.6]
\begin{axis}[
    title={},
    xlabel={$\lambda$},
    ylabel={Value},
    xmin=4, xmax=15,
    ymin=0.15, ymax=0.23,
    xtick={4,6,10,12},
    ytick={0.18, 0.20, 0.23},
    legend pos=north east,
    ymajorgrids=true,
    xmajorgrids=true,
    grid style=dashed,
]

\addplot[
    color=blue,
    mark=square,
    ]
    coordinates {
    (4.16666667,2.23983784e-01)(4.54545455,2.21447731e-01)(5,2.21077612e-01)(5.55555556,2.10669620e-01)(6.25,1.92570342e-01)(7.14285714,1.90400513e-01)(8.33333333,2.02416038e-01)(10,1.90175687e-01)(12.5,1.81992372e-01)(16.66666667,1.50037991e-01)
    };
    \addlegendentry{Empirical risk}
\addplot[
    color=red,
    mark=+,
    ]
    coordinates {
    (4.16666667, 2.23983784e-01)(4.54545455, 2.21447731e-01)(5, 2.21077612e-01)(5.55555556, 2.10669620e-01)(6.25, 2.02416038e-01)(7.14285714,2.02416038e-01)(8.33333333,2.02416038e-01)(10,1.90175687e-01)(12.5,1.81992372e-01)(16.66666667,1.50037991e-01)
    };
    \addlegendentry{Upper bound}
\addplot[
    color=violet,
    mark=star,
    ]
    coordinates {
    (4.16666667, 2.23983784e-01)(4.54545455, 2.21447731e-01)(5, 2.21077612e-01)(5.55555556, 2.10669620e-01)(6.25, 1.92570342e-01)(7.14285714,1.90400513e-01)(8.33333333,1.90400513e-01)(10,1.90175687e-01)(12.5,1.81992372e-01)(16.66666667,1.50037991e-01)
    };
    \addlegendentry{Lower bound}
\end{axis}
\end{tikzpicture}
\caption{}
\label{fig:er_lower_upper}
\end{subfigure}
\caption{a) Sq. bias, variance, risk and (expected) empirical risk behavior in $\lambda$ notation; b) lower $\widetilde{R}_{\lambda}$ and upper $\overline{R}_{\lambda}$ bounds on the empirical risk.}
\end{figure}
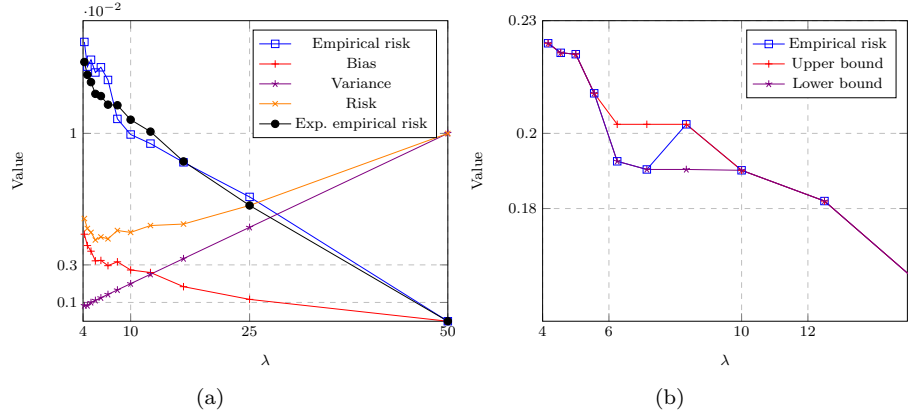

Denote $\widetilde{R}_{\lambda}$ as the tightest non-increasing lower bound on $R_{\lambda}$ and $\overline{R}_{\lambda}$ as the tightest non-increasing upper bound on $R_{\lambda}$. We precise the definitions of the latter quantities below.

\begin{definition} \label{definition_lower_bound_empirical_risk}

Assume that one has the grid of values $\mathfrak{L} = \{1, n/(n-1), \ldots, n \}$, and the empirical risk curve is observed successively, meaning that one starts from $\lambda = 1$ (corresponds to $k = n$) and increases $\lambda$ until the value $n$ (corresponds to $k = 1$). Then, consider the value of $R_{\lambda}$ and its next increment  
$R_{\lambda + \Delta}$ such that $\lambda + \Delta \in \mathfrak{L}$. Define $\widetilde{R}_{1} \coloneqq R_{1}$ and
\begin{equation} \label{lower_bound_empirical_risk}
    \widetilde{R}_{\lambda + \Delta} \coloneqq
    \begin{cases}
    R_{\lambda + \Delta} \textnormal{ if } R_{\lambda + \Delta} - R_{\lambda} \leq 0,\\
    R_{\lambda}, \textnormal{ otherwise; this way, one should wait until }R_{\widetilde{\lambda}} \leq \widetilde{R}_{\lambda + \Delta} \textnormal{ for } \widetilde{\lambda} > \lambda + \Delta.
    \end{cases}    
\end{equation}
\end{definition}

\begin{definition} \label{definition_upper_bound_empirical_risk}

Assume that one has the grid of values $\mathfrak{L} = \{1, n/(n-1), \ldots, n \}$, and the empirical risk curve is observed successively, meaning that one starts from $\lambda = n$ (corresponds to $k = 1$) and decreases $\lambda$ until the value $1$ (corresponds to $k = n$). Then, consider the value of $R_{\lambda}$ and its next increment $R_{\lambda - \Delta}$ such that $\lambda - \Delta \in \mathfrak{L}$. Define $\overline{R}_{n} \coloneqq R_n$ and
\begin{equation} \label{upper_bound_empirical_risk}
    \overline{R}_{\lambda - \Delta} \coloneqq
    \begin{cases}
    R_{\lambda - \Delta} \textnormal{ if } R_{\lambda - \Delta} - R_{\lambda} \geq 0,\\
    R_{\lambda}, \textnormal{ otherwise; this way, one should wait until }R_{\widetilde{\lambda}} \geq \overline{R}_{\lambda - \Delta} \textnormal{ for } \widetilde{\lambda} < \lambda - \Delta.
    \end{cases} 
\end{equation}
\end{definition}
Typical behavior of the defined lower and upper bound $\widetilde{R}_{\lambda}, \overline{R}_{\lambda}$ is illustrated in Fig. \ref{fig:er_lower_upper}. Note that with these definitions:
\begin{align*}
    \lambda_1^{\tau} &= \inf \{ \lambda \in \{ 1, \ldots, n  \} \mid \widetilde{R}_{\lambda} \leq 2 R_2 \},\\
    \lambda_2^{\tau} &= \sup \{ \lambda \in \{ 1, \ldots, n \} \mid \overline{R}_{\lambda} \geq  2 R_2 \}.
\end{align*}
Define an additional stopping rule $\lambda^{\star \star}$ that will be helpful in the analysis.
\begin{equation} \label{k_two_star}
    \lambda^{\star \star} \coloneqq \sup \left\{ \lambda \in \{ 1, \ldots, n \} \mid V(\lambda) + 2 \mathbb{E}_{\varepsilon} R_2 - \sigma^2 + \sigma^2 t_2 \geq B^2(\lambda) \geq V(\lambda) + 2 \mathbb{E}_{\varepsilon} R_2 - \sigma^2 + t_1 \right\}, 
\end{equation}
for some $t_1, t_2 \geq 0$ that will be precised later (see Lemma \ref{deviation_bias_lemma}). If no such $\lambda^{\star \star}$ exists, then $\lambda^{\star \star} = 1$. 
\section{Control of the stochastic part of the variance / the empirical risk} \label{var_control}

\subsection{Control of the stochastic part of the variance}

Consider $v(\lambda_1^{\tau}) = \lVert A_{\lambda_1^{\tau}[k]} \varepsilon \rVert_n^2$ and $V(\lambda_1^{\tau}) = \frac{\sigma^2}{n}\textnormal{tr}\left( A_{\lambda_1^{\tau}[k]} \right)$ for $\lambda_1^\tau[k]$ from Section \ref{main_quantities}. Then for any $\widetilde{t} > 0$,
\begin{equation} \label{series_for_stochastic_term}
    \begin{split}
    \mathbb{P}_{\varepsilon} \left( v(\lambda_1^{\tau}) \geq V(\lambda_1^{\tau}) + \widetilde{t} \right) \leq \mathbb{P}_{\varepsilon} \left( \underset{k \in \{1, \ldots, n \}}{\sup} \left| \lVert A_k \varepsilon \rVert_n^2 - V(k) \right| \geq \widetilde{t}  \right).
    \end{split}
\end{equation}
In what follows, we will bound $\mathbb{P}_{\varepsilon} \left( \underset{k \in \{ 1, \ldots, n \}}{\sup} | \lVert A_k \varepsilon \rVert_n^2 - V(k) | \geq \widetilde{t} \right)$.

\vspace{0.3cm}


Let us define the set of matrices $\overline{\mathcal{A}} \coloneqq \left\{ A_k, \ k = 1, \ldots, n \right\}$, then (\cite[Theorem 3.1]{krahmer2014suprema})
\begin{equation}
    \mathbb{P}_{\varepsilon} \left( \underset{\mathbf{A} \in \overline{\mathcal{A}}}{\sup} \left| \lVert \mathbf{A} \varepsilon \rVert^2 - \mathbb{E}_{\varepsilon} \lVert \mathbf{A}\varepsilon \rVert^2 \right| \geq c_1 E + t \right) \leq 2 \exp \left( -c_2 \min \left( \frac{t^2}{V^2}, \frac{t}{U} \right) \right),
\end{equation}
where 
\begin{align*}
    E &= \gamma_2\left(\overline{\mathcal{A}}, \lVert \cdot \rVert_2\right)\left(\gamma_2(\overline{\mathcal{A}}, \lVert \cdot \rVert_2) + \underset{\mathbf{A} \in \overline{\mathcal{A}}}{\sup}\lVert \mathbf{A} \rVert_{F}\right) + \underset{\mathbf{A} \in \overline{\mathcal{A}}}{\sup}\lVert \mathbf{A} \rVert_{F} \underset{\mathbf{A} \in \overline{\mathcal{A}}}{\sup}\lVert \mathbf{A} \rVert_{2},\\
    U &= \left[ \underset{\mathbf{A} \in \overline{\mathcal{A}}}{\sup}\lVert \mathbf{A} \rVert_{2} \right]^2,\\
    V &= \underset{\mathbf{A} \in \overline{\mathcal{A}}}{\sup}\lVert \mathbf{A} \rVert_{2} \left( \gamma_2(\overline{\mathcal{A}}, \lVert \cdot \rVert_2) + \underset{\mathbf{A} \in \overline{\mathcal{A}}}{\sup}\lVert \mathbf{A} \rVert_{F} \right),
\end{align*}
and $\gamma_2(\overline{\mathcal{A}}, \lVert \cdot \rVert_2)$ can be bounded via the metric entropy of $(\overline{\mathcal{A}}, \lVert \cdot \rVert_2)$ as
\begin{equation*}
    \gamma_2(\overline{\mathcal{A}}, \lVert \cdot \rVert_2) \leq c \int_{0}^{\underset{\mathbf{A} \in \overline{\mathcal{A}}}{\sup}\lVert \mathbf{A} \rVert_{2}} \sqrt{\log N(\overline{\mathcal{A}}; \ \lVert \cdot \rVert_2; \ u)}du.    
\end{equation*}
 
First, notice that, due to Lemma \ref{operator_norm_lemma}, for any $\mathbf{A} \in \overline{\mathcal{A}}$, one has $\lVert \mathbf{A} \rVert_2 \leq c$. Moreover, $\log N(\overline{\mathcal{A}}; \ \lVert \cdot \rVert_2; \ u) \leq \log n$ due to the definition of the metric entropy (see, e.g., \cite[Chapter 5]{wainwright2019high}). These arguments imply 
\begin{align*}
    U &\leq c, \qquad \textnormal{ and }\\
    \gamma_2(\overline{\mathcal{A}}, \lVert \cdot \rVert_2) &\leq c_{\gamma} \sqrt{\log n}.
\end{align*}
Second, as for the Frobenius norm, $\underset{\mathbf{A} \in \overline{\mathcal{A}}}{\sup}\lVert \mathbf{A} \rVert_{F} \leq \sqrt{n}$ due to the definition (\ref{a_k_matrix}). Combining all the pieces together, for any $t > 0$,
\begin{equation*}
    \mathbb{P}_{\varepsilon} \left( \underset{\mathbf{A} \in \overline{\mathcal{A}}}{\sup} \left| \lVert \mathbf{A} \varepsilon \rVert_{n}^2 - \mathbb{E}_{\varepsilon} \lVert \mathbf{A} \varepsilon \rVert_n^2 \right| \geq c_1 \sqrt{\frac{\log n}{n}} + t \right) \leq 2 \exp \left( - c_2 \min \left( n t^2, n t   \right) \right),
\end{equation*}
where $c_1$ and $c_2$ may depend on $\sigma^2$. Take $\widetilde{t} = c_1 \sqrt{\frac{\log n}{n}} + t$ in (\ref{series_for_stochastic_term}), then for any $t > 0$,
\begin{equation*}
    \mathbb{P}_{\varepsilon} \left( v(\lambda_1^{\tau}) \geq V(\lambda_1^{\tau}) + c_1 \sqrt{\frac{\log n}{n}} + t \right) \leq 2 \exp \left( - c n \min \left( t^2, t \right) \right).
\end{equation*}
%
%
%
%
\subsection{Control of the empirical risk around its expectation} \label{control_emp_risk_subsect}

Define the quadratic form $Q_k(Y) = \frac{1}{n}Y^{\top} (I_n - A_{k})^{\top} (I_n - A_k)Y$ where $Y \overset{\textnormal{i.i.d.}}{\sim} \mathcal{N}(F^*, \Sigma)$ with $\Sigma = \sigma^2 I_n$.

Define $\mathcal{L} = \Sigma^{-1/2}Y$ and $\mathcal{J} = \mathcal{L} - \Sigma^{-1/2}F^*$, then
\begin{align*}
    Q_k(Y) &= \frac{1}{n}(\mathcal{J} + \Sigma^{-1/2}F^*)^{\top}\Sigma^{1/2}(I_n - A_k)^{\top}(I_n - A_k)\Sigma^{1/2}(\mathcal{J} + \Sigma^{-1/2}F^*) \\
    &= \frac{\sigma^2}{n}(\mathcal{J} + \Sigma^{-1/2}F^*)^{\top}P^{\top}\Lambda P (\mathcal{J} + \Sigma^{-1/2}F^*)\\
    &= \frac{\sigma^2}{n}(P\mathcal{J} + P\Sigma^{-1/2}F^*)^{\top}\Lambda(P \mathcal{J} + P\Sigma^{-1/2}F^*)\\
    =& \frac{\sigma^2}{n} (u + b)^{\top}\Lambda (u + b) = \frac{\sigma^2}{n} \sum_{i=1}^n \Lambda_i (u_i + b_i)^2.
\end{align*}
where $b = \frac{1}{\sigma}P F^*$.
It implies that
\begin{equation}
    R_k - \mathbb{E}_{\varepsilon}R_k = \underbrace{\frac{\sigma^2}{n}\sum_{i=1}^n \Lambda_i \left( u_i^2 - 1 \right)}_{\mathcal{A}} + \underbrace{\frac{2 \sigma^2}{n}\sum_{i=1}^n \Lambda_i b_i u_i}_{\mathcal{B}}
\end{equation}
where $u_i \overset{\textnormal{i.i.d.}}{\sim} \mathcal{N}(0, 1)$.

Define $\phi_i = \frac{\sigma^2}{n}\Lambda_i$. Further, we will use (\cite[Lemma 1]{laurent2000adaptive}), i.e. for any $t > 0$,
\begin{align*}
    \mathbb{P}_{\varepsilon}\left( \mathcal{A} \geq 2 \lVert \phi \rVert_2 \sqrt{t} + 2 \lVert \phi \rVert_{\infty} t  \right) &\leq \exp(-t),\\
    \mathbb{P}_{\varepsilon}\left( \mathcal{A} \leq -2 \lVert \phi \rVert_2\sqrt{t} \right) &\leq \exp(-t).
\end{align*}
First, since for the eigenvalues of $(I_n - A_k)^{\top}(I_n - A_k)$: $\frac{\Lambda_i}{4} \leq 1$ for $i = 1, \ldots, n$, then
\begin{align*}
    \lVert \phi \rVert_{\infty} &\leq \frac{4 \sigma^2}{n},\\
    \lVert \phi \rVert_2^2 &\leq \frac{16 \sigma^4}{n}.
\end{align*}
Thus, on the one hand, for any $t > 0$,
\begin{align*}
    \mathbb{P}_{\varepsilon}\left( \mathcal{A} \geq 8 \sigma^2 \sqrt{t} + 8 \sigma^2 t \right) &\leq \exp(-nt),\\
    \mathbb{P}_{\varepsilon}\left( \mathcal{A} \leq - 8 \sigma^2 \sqrt{t} \right) &\leq \exp(-nt).
\end{align*}
And on the other hand, for any $t > 0$,
\begin{equation*}
    \mathbb{P}_{\varepsilon}\left( \mid \mathcal{B} \mid \geq t \right) \leq 2 \exp \left( - \frac{n t^2 \sigma^2}{128 \mathcal{M}^2} \right).
\end{equation*}
It implies that for any $t > 0$,
\begin{equation} \label{final_sup_for_empirical_risk}
    R_k - \mathbb{E}_{\varepsilon}R_k \in \left( -9 \sigma^2 t, 8 \sigma^2 t^2 + 9\sigma^2 t \right)
\end{equation}
with probability at least $1 - 4 \exp \left( - \min \left(1, \frac{\sigma^2}{128 \mathcal{M}^2}  \right) n t^2 \right)$.

\section{Deviation inequality for the variance term} \label{appendix_variance_deviation}

This is the first deviation inequality for $\lambda_1^{\tau}$ that will be used to control the variance term.

\begin{lemma} \label{lemma_for_variance}
Under Assumption \ref{assumption_boundness}, define $\mathcal{K}_{V} \subseteq \{1, \ldots, n \}$ such that, for any $\lambda \in \mathcal{K}_{V}$, one has $V(\lambda) \geq B^2(\lambda) + \sigma^2 - 2 \mathbb{E}_{\varepsilon} R_2 + t$ for some $t \geq 0$. Recall the definition of $\lambda_1^{\tau}$ from Eq. (\ref{stopping_times_tilde}), then for any $\lambda \in \mathcal{K}_V$,
\begin{equation}
    \mathbb{P}_{\varepsilon} \left( \lambda_1^{\tau} > \lambda \right) \leq 4 \exp \left( -c n \min \left( \frac{t^2}{\sigma^4}, \frac{t^2}{\sigma^2}, \frac{t}{\sigma^2} \right) \right),
\end{equation}
where constant $c$ depends on $\mathcal{M}$.
\end{lemma}

\begin{proof}
We start with the following series of inequalities that can be derived from the definition of $\lambda_1^{\tau}$ and the lower bound on the empirical risk $\widetilde{R}_{\lambda}$ (see Section \ref{main_quantities}). Since for any $\widetilde{t} \geq 0$, $\mathbb{P}_{\varepsilon}\left( Q_1 + Q_2 > \widetilde{t} \right) \leq \mathbb{P}_{\varepsilon}\left( Q_1 > \widetilde{t}/2 \right) + \mathbb{P}_{\varepsilon}\left( Q_2 > \widetilde{t}/2 \right)$, with $Q_1 = R_{\lambda} - \mathbb{E}_{\varepsilon}R_{\lambda}$ and $Q_2 = 2\mathbb{E}_{\varepsilon}R_2 - 2R_2 $, one has
\begin{align*}
    \mathbb{P}_{\varepsilon} \left( \lambda_1^{\tau} > \lambda \right) &= \mathbb{P}_{\varepsilon}\left( \widetilde{R}_{\lambda} > 2 R_2 \right)\\
    &\leq \mathbb{P}_{\varepsilon} \left( R_{\lambda} - \mathbb{E}_{\varepsilon}R_{\lambda} > 2 \mathbb{E}_{\varepsilon}  R_2 - \mathbb{E}_{\varepsilon}R_{\lambda} + 2R_2 - 2\mathbb{E}_{\varepsilon}R_2 \right) \\
    &\leq \mathbb{P}_{\varepsilon}\left( R_{\lambda} - \mathbb{E}_{\varepsilon}R_{\lambda} > \frac{1}{2}[2\mathbb{E}_{\varepsilon}R_2 - \mathbb{E}_{\varepsilon}R_{\lambda}] \right) + \mathbb{P}_{\varepsilon}\left( R_2 - \mathbb{E}_{\varepsilon}R_2 < - \frac{1}{4}[2\mathbb{E}_{\varepsilon}R_2 - \mathbb{E}_{\varepsilon}R_{\lambda}] \right).
\end{align*}
Due to Eq. (\ref{exp_empirical_risk}), one has 
\begin{equation} \label{deviation_to_extend}
    2 \mathbb{E}_{\varepsilon}R_2 - \mathbb{E}_{\varepsilon}R_{\lambda} = V(\lambda) - B^2(\lambda) + 2 \mathbb{E}_{\varepsilon} R_2 - \sigma^2 \geq t.
\end{equation}
%
Moreover,
\begin{align*}
        R_{\lambda} - \mathbb{E}_{\varepsilon}R_{\lambda} &= \lVert (I_n - A_{\lambda[k]})\varepsilon \rVert_n^2 - \frac{\sigma^2}{n}\left( n - \textnormal{tr}(A_{\lambda[k]}) \right) + 2 \langle (I_n - A_{\lambda[k]}) F^*, (I_n - A_{\lambda[k]})\varepsilon \rangle_n.
\end{align*}
Define for simplicity $M_{\lambda[k]} \coloneqq I_n - A_{\lambda[k]}$, then %
\begin{align*}
    \mathbb{P}_{\varepsilon}\left( \lambda_1^{\tau} > \lambda \right) &\leq \mathbb{P}_{\varepsilon} \left( \lVert M_{\lambda[k]} \varepsilon \rVert_n^2 - \frac{\sigma^2}{n}\left( n - \textnormal{tr}(A_{\lambda[k]}) \right) > \frac{t}{4} \right) + \mathbb{P}_{\varepsilon} \left( 2 \langle M_{\lambda[k]} F^*, M_{\lambda[k]} \varepsilon \rangle_n > \frac{t}{4} \right)\\
    &+ \mathbb{P}_{\varepsilon}\left(\lVert M_2 \varepsilon \rVert_n^2 - \frac{\sigma^2}{n}\left( n - \textnormal{tr}(A_2) \right) < -\frac{t}{8} \right) + \mathbb{P}_{\varepsilon} \left( 2 \langle M_2 F^*, M_2 \varepsilon \rangle_n < -\frac{t}{8} \right).
\end{align*}
Further, we will concentrate the quadratic and linear terms as follows.

\textbf{First term.} The linear term $2 \langle M_{\lambda[k]} F^*, M_{\lambda[k]} \varepsilon \rangle_n$: using Lemma 
\ref{linear_term_concentration} and Lemma \ref{operator_norm_lemma} gives us
\begin{align*}
    \mathbb{P}_{\varepsilon}\left( 2 \langle M_{\lambda[k]} F^*, M_{\lambda[k]} \varepsilon \rangle_n > \frac{t}{4} \right) &= \mathbb{P}_{\varepsilon} \left( \langle M_{\lambda[k]}^\top M_{\lambda[k]} F^*, \varepsilon \rangle > \frac{n t}{8} \right)\\
    &\leq \exp \left[ - \frac{n^2 t^2}{128 \sigma^2 \lVert M_{\lambda[k]}^\top M_{\lambda[k]} F^* \rVert^2} \right]\\
    &\leq \exp \left[ - \frac{n t^2}{128 \sigma^2 \lVert M_{\lambda[k]}^\top M_{\lambda[k]} \rVert_2^2 \lVert f^* \rVert_n^2 } \right]\\
    &\leq \exp \left[ - \frac{n t^2}{512 \mathcal{M}^2 \sigma^2} \right].
\end{align*}
\textbf{Second term.} Consider the quadratic term $\lVert M_{\lambda[k]} \varepsilon \rVert_n^2 - \frac{\sigma^2}{n} \left( n - \textnormal{tr}A_{\lambda[k]} \right)$: combining Lemma \ref{quadratic_term_concentration} and Lemma \ref{operator_norm_lemma} gives
\begin{align*}
    \mathbb{P}_{\varepsilon} \left( \lVert M_{\lambda[k]} \varepsilon \rVert_n^2 - \frac{\sigma^2}{n}\left( n - \textnormal{tr}A_{\lambda[k]} \right) > \frac{t}{4} \right) &\leq \exp \bigg[-c \min \bigg( \frac{n^2 t^2}{16\sigma^4 \lVert M_{\lambda[k]}^\top M_{\lambda[k]} \rVert_{F}^2},\\ & \frac{n t}{4 \sigma^2 \lVert M_{\lambda[k]}^\top M_{\lambda[k]} \rVert_{2}} \bigg) \bigg] \\
    &\leq \exp \left[ - c \min \left( \frac{n t^2}{16 \sigma^4}, \frac{nt}{4 \sigma^2} \right) \right],
\end{align*}
where constant $c$ is numeric.  
Finally, 
\begin{align*}
    \mathbb{P}_{\varepsilon}\left( 2 \langle M_2F^*, M_2\varepsilon \rangle_n < -\frac{t}{8} \right) \leq \text{exp}\left[ -\frac{n^2 t^2}{512 \sigma^2 \lVert M_{\lambda[k]}^\top M_{\lambda[k]}F^* \rVert^2}  \right] &\leq \text{exp}\left[ - \frac{n t^2}{8192 \sigma^2 \mathcal{M}^2} \right], \\
    \mathbb{P}_{\varepsilon}\left( \lVert M_{2} \varepsilon \rVert_n^2 - \frac{\sigma^2}{n} (n - \text{tr}A_{2}) < -\frac{t}{8} \right) &\leq \text{exp}\left[ -c \text{min}\left( \frac{nt^2}{64\sigma^4}, \frac{n t}{8 \sigma^2} \right) \right],
\end{align*}
where constant $c$ is numeric. 

\end{proof}

Based on Lemma \ref{lemma_for_variance}, due to the fact that the variance $V(\lambda)$ is increasing  w.r.t. $\lambda \in \{1, \ldots, n\}$, the following corollary holds.
\begin{corollary} \label{corollary_for_variance}
Define $[\widetilde{\lambda} + 1]$ as the next value of $\lambda > \widetilde{\lambda}$ on the grid $\{ 1, \ldots, n \}$. Then for any $t > 0$, define $0 \leq \Delta t \leq t$ as the distance between $V([\lambda_2^* + 1]) + t$ and $V(\lambda_0)$, where $V(\lambda_0)$ is the closest to $V([\lambda_2^* + 1]) + t$ value of $V(\lambda)$ such that 
\begin{itemize}
    \item either $t - \Delta t$ is lower than or equal to $V(\lambda) - B^2(\lambda) + 2\mathbb{E}_{\varepsilon}R_2 - \sigma^2$ while $V(\lambda) + 2\mathbb{E}_{\varepsilon}R_2 - \sigma^2 \geq B^2(\lambda)$,
    \item or $V(\lambda)$ is lower than or equal to $V([\lambda_2^* + 1]) + t$ while $B^2(n) = 0$,
\end{itemize}
over the grid of $\lambda \in \{[\lambda_2^* + 1], [\lambda_2^* + 2], \ldots, n \}$. Then due to the monotonicity of the variance term,
\begin{equation}
    \mathbb{P}_{\varepsilon}\left( V(\lambda_1^{\tau}) > V([\lambda_2^* + 1]) + t - \Delta t \right) \leq 4 \exp \left( -c n \min \left( \frac{(t - \Delta t)^2}{\sigma^4}, \frac{(t - \Delta t)^2}{\sigma^2}, \frac{t - \Delta t}{\sigma^2} \right) \right)
\end{equation}
for constant $c$ that depends only on $\mathcal{M}$. Moreover, due to Lemma \ref{variance_aux_lemma}, $\frac{1}{2}V([\lambda_2^* + 1]) \leq V(\lambda_2^*) \leq V([\lambda_2^* + 1])$, which implies that
\begin{equation*}
    \mathbb{P}_{\varepsilon} \left( V(\lambda_1^{\tau}) > 2 V(\lambda_2^*) + t - \Delta t \right) \leq 4 \exp \left( -c n \min \left(  \frac{(t - \Delta t)^2}{\sigma^4}, \frac{(t - \Delta t)^2}{\sigma^2}, \frac{t - \Delta t}{\sigma^2} \right) \right) \ \forall t > 0.
\end{equation*}
\end{corollary}
Thus, one is able to control $V(\lambda_1^{\tau})$ via $V(\lambda_2^*)$.

\section{Deviation inequality for the bias term} \label{appendix_bias_deviation}

What follows is the second deviation inequality for $\lambda_1^{\tau}$ that will be further used to control the bias term.
\begin{lemma} \label{left_deviation_ineq}
Under Assumption \ref{assumption_boundness}, define $\mathcal{K}_{B} \subseteq \{ 1, \ldots, n \}$ such that, for any $\lambda \in \mathcal{K}_{B}$, one has $V(\lambda) + 2\mathbb{E}_{\varepsilon}R_2 - \sigma^2 + \sigma^2 t_2 \geq B^2(\lambda) \geq V(\lambda) + 2\mathbb{E}_{\varepsilon}R_2 - \sigma^2 + t_1$ for some $t_1, t_2 \geq 0$. Then if $\mathcal{K}_B$ is not empty, $\lambda_1^{\tau}$ from Eq. (\ref{stopping_times_tilde}) satisfies 
\begin{align*}
    \mathbb{P}_{\varepsilon} \left( \lambda_1^{\tau} \leq \lambda \right) &\leq 8 \exp \left( - c n \min \left( \frac{t_1^2}{\sigma^4}, \frac{t_1^2}{\sigma^2}, \frac{t_1}{\sigma^2} \right) \right) + 5 \exp \left( - \min \left(1, \frac{\sigma^2}{128 \mathcal{M}^2} \right) n t_2^2 \left( 1 - \frac{1}{n^{1.5}2^{n/2}} \right)^2 \right) \\
    &+ \frac{5 (t_2 + t_2^2)}{\sigma^2 \sqrt{ n \left( \frac{n}{2} - 1 \right)}},
\end{align*}
where constant $c$ depends on $\mathcal{M}$ only, and $n \geq 3$.
\end{lemma}

\begin{proof}
Consider Ineq. (\ref{final_sup_for_empirical_risk}) and the event
\begin{equation*}
     \left\{ -9 \sigma^2 t < R_{\lambda} - \mathbb{E}_{\varepsilon}R_{\lambda} < 8 \sigma^2 t^2 + 9 \sigma^2 t  \right\}
\end{equation*}
for any $t > 0$.
Further, $\lambda_1^{\tau} \leq \lambda_2^{\tau}$ and recall that $\overline{R}_{\lambda}$ is the upper bound on $R_{\lambda}$ from Section \ref{main_quantities}, which implies that
\begin{align} \label{first_split}
    \begin{split}
    \mathbb{P}_{\varepsilon} \left( \lambda_1^{\tau} \leq \lambda \right) &= \underbrace{\mathbb{P}_{\varepsilon} \left( \left\{ \lambda_1^{\tau} \leq \lambda \right\} \bigcap \left\{ \lambda > \lambda_2^{\tau} \right\}  \right)}_{\mathcal{A}} + \underbrace{\mathbb{P}_{\varepsilon} \left( \left\{ \lambda_1^{\tau} \leq \lambda \right\} \bigcap \left\{ \lambda \leq \lambda_2^{\tau} \right\} \right)}_{\mathcal{B}},\\
    \mathcal{A} &= \mathbb{P}_{\varepsilon}\left(\overline{R}_{\lambda} < 2 R_2 \right) \leq \mathbb{P}_{\varepsilon} \left( R_{\lambda} < 2R_2 \right),\\
    \mathcal{B} &= \mathbb{P}_{\varepsilon} \left( \lambda \in \left[\lambda_1^{\tau}, \lambda_2^{\tau}\right] \right).
    \end{split}
\end{align}
Consider the probability $\mathcal{B}$ from (\ref{first_split}). With probability at least $1 - 4 \text{exp}\left( - n t_2^2 \min \left( 1, \sigma^2 / 128 \mathcal{M}^2 \right) \right)$, 
\begin{equation*}
    8 \sigma^2 t_2^2 + 9 \sigma^2 t_2 > R_{\lambda} - \mathbb{E}_{\varepsilon}R_{\lambda} > - 9 \sigma^2 t_2
\end{equation*}
for any $t_2 > 0$. Denote this high probability event as $\mathcal{E}$. Then 
\begin{align*}
    \mathcal{B} &= \mathbb{P}_{\varepsilon} \left( \{ \lambda \in [ \lambda_1^{\tau}, \lambda_2^{\tau} ] \} \cap \{ R_{\lambda} < 2R_2 \} \right) + \mathbb{P}_{\varepsilon}\left( \{ \lambda \in [\lambda_1^{\tau}, \lambda_2^{\tau}] \} \cap \{ R_{\lambda} \geq 2 R_2 \}  \right)\\
    &\leq \mathbb{P}_{\varepsilon} \left( R_{\lambda} < 2R_2 \right) + \mathbb{P}_{\varepsilon} \left( \{R_{\lambda} \geq 2R_2\} \cap \mathcal{E} \right) + \mathbb{P}_{\varepsilon} \left( \{R_{\lambda} \geq 2R_2\} \cap \mathcal{E}^c \right).
\end{align*}
%
On the one hand, since for any $\widetilde{t} \geq 0$, $ \ \mathbb{P}_{\varepsilon}\left( Q_1 + Q_2 < -\widetilde{t} \right) \leq \mathbb{P}_{\varepsilon}\left(Q_1 < -\widetilde{t}/2 \right) + \mathbb{P}_{\varepsilon}\left(Q_2 < -\widetilde{t}/2 \right)$, with $Q_1 = R_{\lambda} - \mathbb{E}_{\varepsilon}R_{\lambda}$ and $Q_2 = 2\mathbb{E}_{\varepsilon}R_2 - 2R_2$, one gets
\begin{align*}
    \mathbb{P}_{\varepsilon} \left( R_{\lambda} < 2 R_2  \right) &= \mathbb{P}_{\varepsilon}\left( R_{\lambda} - \mathbb{E}_{\varepsilon}R_{\lambda} < - [B^2(\lambda) - V(\lambda) + \sigma^2 - 2 \mathbb{E}_{\varepsilon} R_2] + 2R_2 - 2 \mathbb{E}_{\varepsilon}R_2 \right)\\
    &\leq \mathbb{P}_{\varepsilon}\left( R_{\lambda} - \mathbb{E}_{\varepsilon}R_{\lambda} < - \frac{B^2(\lambda) - V(\lambda) + \sigma^2 - 2\mathbb{E}_{\varepsilon}R_2}{2} \right)\\ &+ \mathbb{P}_{\varepsilon}\left( 2 \left( \mathbb{E}_{\varepsilon}R_2 - R_2 \right) < -\frac{B^2(\lambda) - V(\lambda) + \sigma^2 - 2 \mathbb{E}_{\varepsilon}R_2}{2} \right)\\
    &\leq \mathbb{P}_{\varepsilon}\left( R_{\lambda} - \mathbb{E}_{\varepsilon}R_{\lambda} < - \frac{t_1}{2} \right) + \mathbb{P}_{\varepsilon}\left( R_2 - \mathbb{E}_{\varepsilon}R_2 > \frac{t_1}{4} \right).
\end{align*}
On the other hand, for $0 < z < 1$, using the SVD decomposition from Subsection \ref{control_emp_risk_subsect},
\begin{align*}
    \mathbb{P}_{\varepsilon} \left( \{R_{\lambda} \geq 2R_2 \} \cap \mathcal{E} \right) &\leq \mathbb{P}_{\varepsilon} \left( R_{\lambda} - \mathbb{E}_{\varepsilon}R_{\lambda} \in \left[ \underbrace{- \sigma^2 t_2}_{a_1}, \underbrace{9 \sigma^2 t_2 + 8 \sigma^2 t_2^2}_{a_2}  \right] \right)\\
    &\leq \mathbb{P}_u\left( \frac{4\sigma^2}{n}\sum_{i=1}^n \frac{\Lambda_i}{4} (u_i^2 - 1) \in \left[a_1 (1 - z), a_2 (1 - z)\right] \right) \\
    &+\mathbb{P}_u \left(\frac{2\sigma^2}{n}\sum_{i=1}^n \Lambda_i b_i u_i \in \left[a_1 z, a_2 z \right] \right).
\end{align*}
First, by using Markov's inequality: for any $q > 0$
\begin{equation*}
    \mathbb{E}X = \int_{0}^{+\infty}w f(w) dw = \int_{0}^{e^{q a_1 z}}w f(w) dw + \int_{e^{q a_1 z}}^{e^{q a_2 z}}w f(w)dw + \int_{e^{q a_2 z}}^{+\infty} w f(w) dw \geq \int_{e^{q a_1 z}}^{e^{q a_2 z}} e^{q a_1 z} f(w)dw.
\end{equation*}
Define $X = \exp \left( q \frac{2 \sigma^2}{n} \sum_{i=1}^n \Lambda_i b_i u_i \right)$. Thus, 
\begin{align*}
    p_1 = \mathbb{P}_u \left( \frac{2 \sigma^2}{n}\sum_{i=1}^n \Lambda_i b_i u_i \in [a_1 z, a_2 z ] \right) &\leq \exp \left( - q a_1 z \right) \mathbb{E}_u \left[ \exp \left( q \frac{2 \sigma^2}{n}\sum_{i=1}^n \Lambda_i b_i u_i  \right) \right] \\
    &\leq \exp \left( -q a_1 z + \frac{q^2}{2}\sum_{i=1}^n \frac{4 \sigma^4}{n^2} \Lambda_i^2 b_i^2 \right).
\end{align*}
Minimizing the above expression over $q > 0$, gives us $q = \frac{a_1 z}{4 \sigma^4 / n^2 \sum_{i=1}^n \Lambda_i^2 b_i^2}$. It implies
\begin{align*}
p_1 &\leq \exp \left( - \frac{a_1^2 z^2}{\frac{8\sigma^4}{n^2}\sum_{i=1}^n \Lambda_i^2 b_i^2} \right) \\
&\leq \exp \left( - \frac{t_2^2 z^2 n \sigma^2}{128 \mathcal{M}^2} \right).
\end{align*}
Second, let us consider
\begin{equation*}
    p_2 = \mathbb{P}_{\varepsilon} \left( \frac{4\sigma^2}{n}\sum_{i=1}^n \frac{\Lambda_i}{4} (u_i^2 - 1) \in \left[ a_1(1 - z), a_2(1 - z) \right] \right).
\end{equation*}
Notice that $\sum_{i=1}^n \frac{\Lambda_i}{4} u_i^2$ is distributed as generalized chi square. Let us denote its probability density function as $f$, then by using (\cite[Eq. (68)]{kotz1967series})
\begin{align*}
    f(t) &= \sum_{j=0}^{\infty} c_j^{(L)} f_{\chi_n^2}(t/\beta) j! \Gamma\left(\frac{n}{2}\right)\left[ \beta \Gamma \left( \frac{n}{2} + j \right) \right]^{-1} L_j^{(n/2 - 1)}\left( \frac{t}{2 \beta} \right)\\
    &\leq \frac{1}{\beta}f_{\chi_n^2}(t/\beta) (1 - R)^{-n/2}e^{\frac{t}{4\beta}}\left( 1 - \frac{\epsilon}{R} \right)^{-1}
\end{align*}
for any $\beta > \frac{\Lambda_1}{8}$ and $0 < \epsilon < R < 1$, where $L_j^{(n/2 - 1)}\left( \frac{t}{2\beta} \right)$ is a Laguerre polynomial and $f_{\chi_n^2}\left( t/\beta \right)$ is the p.d.f. of the $\chi_n^2$ random variable.

Recall that $f_{\chi_n^2}(t/\beta) = \frac{1}{2^{n/2} \Gamma \left( \frac{n}{2} \right)} \left( \frac{t}{\beta} \right)^{\frac{n}{2} - 1}e^{-\frac{t}{2 \beta}}$, therefore
\begin{equation*}
    f(t) \leq \frac{1}{\beta}\frac{1}{2^{n/2}\Gamma (\frac{n}{2})} \left( \frac{t}{\beta} \right)^{\frac{n}{2} - 1} e^{-\frac{t}{4\beta}}\frac{1}{(1 - R)^{n/2}(1 - \frac{\epsilon}{R})}.
\end{equation*}
One can conclude that
\begin{equation*}
    p_2 \leq \frac{1}{\beta}\frac{1}{2^{n/2}\Gamma (\frac{n}{2})}\frac{10 (t_2 + t_2^2) (1 - z)n}{(1 - R)^{n/2}(1 - \frac{\epsilon}{R})} \underset{t \in [a_1 n (1 - z) / \sigma^2 + \text{tr}(\Lambda/4), a_2 n (1 - z) / \sigma^2 + \text{tr}(\Lambda/4)]}{\max}\left[ \left( \frac{t}{\beta} \right)^{\frac{n}{2} - 1} \exp \left( - \frac{t}{4 \beta} \right)  \right].
\end{equation*}
The function $\left( \frac{t}{\beta}\right)^{\frac{n}{2} - 1} \exp \left( - \frac{t}{4 \beta} \right)$ is maximized at $t = 4 \beta \left( \frac{n}{2} - 1 \right)$, so by using $\Gamma \left( \frac{n}{2} \right) = \left( \frac{n}{2} - 1 \right)!$ and the Stirling's formula, with $\frac{\Lambda_1}{8} < \beta < 1$, one gets
\begin{align*}
    p_2 &\leq \frac{1}{\beta 2^{n/2}}\frac{10 \cdot 2^{n - 2}  e^{-\left(\frac{n}{2}  - 1\right)} \left( \frac{n}{2} - 1 \right)^{\frac{n}{2} - 1} (t_2 + t_2^2) (1 - z) n}{\sqrt{n/2 - 1} \left( \frac{n/2 - 1}{e} \right)^{n/2 - 1} \exp\left( \frac{1}{12 (n/2 - 1)} - \frac{1}{360 (n/2 - 1)^3} \right)(1 - R)^{n/2}(1 - \frac{\epsilon}{R})}\\
    &\leq \frac{10(t_2 + t_2^2)(1 - z)n2^{n-2}}{\beta 2^{n/2} \sqrt{\frac{n}{2} - 1}(1 - R)^{n/2}(1 - \frac{\epsilon}{R})}\\
    &= \frac{10(t_2 + t_2^2)(1 - z)n 2^{n/2}}{4 \beta \sqrt{\frac{n}{2} - 1} \left( 2 \left(\frac{1-R}{2}\right) \right)^{n/2}\left( 1 - \frac{\epsilon}{R} \right)}.
\end{align*}
Let us take $\epsilon = \frac{R}{2}$ and $z = 1 - \frac{\beta}{n^{1.5}}\left( \frac{1 - R}{2} \right)^{\frac{n}{2}}$. Combining the upper bounds for $p_1$ and $p_2$, we have 
\begin{align*}
    \mathbb{P}_{\varepsilon}\left( \{ R_{\lambda} \geq 2R_2\} \cap \mathcal{E} \right) &\leq \exp \left( - \frac{n \sigma^2 t_2^2 \left[1 - \frac{\beta}{n^{1.5}}\left( \frac{1 - R}{2} \right)^{n/2} \right]^2}{128 \mathcal{M}^2} \right) + \frac{5 (t_2 + t_2^2)}{ \sqrt{n \left( \frac{n}{2} - 1 \right)} }.
\end{align*}
After that, 
\begin{align*}
    \mathbb{P}_{\varepsilon} \left( \{ R_{\lambda} \geq 2R_2 \} \cap \mathcal{E}^c \right) \leq 4 \exp \left( - n t_2^2 \min \left( 1, \frac{\sigma^2}{128 \mathcal{M}^2} \right) \right).
\end{align*}
Thus, applying $\frac{\Lambda_1}{8} < \beta < 1$ and combining all together, 
\begin{align*}
    \mathbb{P}_{\varepsilon}\left( \lambda_1^{\tau} \leq \lambda \right) &\leq 2 \underbrace{\left[\mathbb{P}_{\varepsilon}\left( R_{\lambda} - \mathbb{E}_{\varepsilon}R_{\lambda} < -\frac{t_1}{2} \right) + \mathbb{P}_{\varepsilon}\left( R_2 - \mathbb{E}_{\varepsilon}R_2 > \frac{t_1}{4} \right) \right] }_{\mathcal{C}} \\ &+ 5 \exp \left( - \min \left( 1, \frac{\sigma^2}{128 \mathcal{M}^2} \right) n t_2^2 \left[1 - \frac{\beta}{n^{1.5}}\left( \frac{1 - R}{2} \right)^{n/2} \right]^2 \right) + \frac{5 (t_2 + t_2^2)}{\sqrt{n \left( \frac{n}{2} - 1 \right)}} \\
    &\leq 2\mathcal{C} + 5\exp \left( - \min \left(1, \frac{\sigma^2}{128\mathcal{M}^2} \right) n t_2^2 \left[ 1 - \frac{1}{n^{1.5} 2^{n/2}} \right]^2 \right) + \frac{5 (t_2 + t_2^2)}{\sqrt{n \left( \frac{n}{2} - 1 \right)}},
\end{align*}
where $t_1 = B^2(\lambda) - V(\lambda) + \sigma^2 - 2\mathbb{E}_{\varepsilon}R_2$ and $\sigma^2 t_2 =  V(\lambda) - B^2(\lambda) - \sigma^2 + 2\mathbb{E}_{\varepsilon}R_2$.
Since 
\begin{equation*}
R_{\lambda} - \mathbb{E}_{\varepsilon}R_{\lambda} = \lVert (I_n - A_{\lambda[k]}) \varepsilon \rVert_n^2 - \frac{\sigma^2}{n}(n - \textnormal{tr}(A_{\lambda[k]})) + 2 \langle (I_n - A_{\lambda[k]})F^*, (I_n - A_{\lambda[k]})\varepsilon \rangle_n,    
\end{equation*}
we have 
\begin{align*}
    \mathcal{C} &\leq \mathbb{P}_{\varepsilon}\left( \lVert M_{\lambda[k]}\varepsilon \rVert_n^2 - \frac{\sigma^2}{n}(n - \text{tr}(A_{\lambda[k]})) < - \frac{t_1}{4} \right) + \mathbb{P}_{\varepsilon}\left( 2 \langle M_{\lambda[k]}F^*, M_{\lambda[k]}\varepsilon \rangle_n < - \frac{t_1}{4} \right)\\
    &+ \mathbb{P}_{\varepsilon} \left( \lVert M_2 \varepsilon \rVert_n^2 - \frac{\sigma^2}{n}(n - \text{tr}(A_2)) > \frac{t_1}{8} \right) + \mathbb{P}_{\varepsilon} \left(2 \langle M_2 F^*, M_2 \varepsilon \rangle_n > \frac{t_1}{8} \right),
\end{align*}
where the matrix $M_{\lambda[k]} = I_n - A_{\lambda[k]}$.
Finally, 
\begin{align*}
    \mathbb{P}_{\varepsilon}\left( 2 \langle M_{\lambda[k]}F^*, M_{\lambda[k]}\varepsilon \rangle_n < -\frac{t_1}{4} \right) \leq \text{exp}\left[ -\frac{n^2 t_1^2}{128 \sigma^2 \lVert M_{\lambda[k]}^\top M_{\lambda[k]}F^* \rVert^2}  \right] &\leq \text{exp}\left[ - \frac{n t_1^2}{128 c \sigma^2 \lVert f^* \rVert_n^2} \right], \\
    \mathbb{P}_{\varepsilon}\left( \lVert M_{\lambda[k]} \varepsilon \rVert_n^2 - \frac{\sigma^2}{n} (n - \text{tr}A_{\lambda[k]}) < -\frac{t_1}{4} \right) &\leq \text{exp}\left[ -c \text{min}\left( \frac{nt_1^2}{16\sigma^4}, \frac{n t_1}{4 \sigma^2} \right) \right],
\end{align*}
where constant $c$ is numeric. The same inequalities (with different numeric constants) hold true for the matrix $M_2$.

\end{proof}

\begin{lemma} \label{deviation_bias_lemma}
Under Assumption \ref{assumption_boundness}, recall the definitions of $\lambda_1^{\tau}$ and $\lambda_2^*$ from Eq. (\ref{stopping_times_tilde}) and Lemma \ref{left_deviation_ineq}. Then for any $t \geq 0$ and $t_2 \geq \frac{4 \mathcal{M}^2}{\sigma^2}$,
\begin{equation}
    B^2(\lambda_1^{\tau}) < 2  V(\lambda_2^*) + 2\mathbb{E}_{\varepsilon}R_2 - \sigma^2 + 2t  
\end{equation}
with probability at least $1 - 12 \exp \left(- c n \min \left( \frac{t^2}{\sigma^2}, \frac{t^2}{\sigma^4}, \frac{t}{\sigma^2} \right) \right) - 5 \exp \left( - \min \left(1, \frac{\sigma^2}{128 \mathcal{M}^2} \right) n t_2^2 \left( 1 - \frac{1}{n^{1.5}2^{n/2}} \right)^2 \right) - \frac{5 (t_2 + t_2^2)}{ \sqrt{n \left( \frac{n}{2} - 1 \right)} }$, where  constant $c$ depends on  $\mathcal{M}$ and $n \geq 3$.
\end{lemma}

\begin{proof}
Consider the event $\mathcal{E}(\lambda)$ from Lemma \ref{left_deviation_ineq} for each $\lambda \in \mathcal{K}_{B}$. Then,
\begin{align*}
    \mathbb{P}_{\varepsilon} \left( \mathcal{E}(\lambda) \right) &\leq 8 \exp \left( - c n \min \left( \frac{v_1^2}{\sigma^4}, \frac{v_1^2}{\sigma^2}, \frac{v_1}{\sigma^2} \right) \right) + 5 \exp \left( - \min \left(1, \frac{\sigma^2}{128 \mathcal{M}^2} \right) n v_2^2 \left( 1 - \frac{1}{n^{1.5}2^{n/2}} \right)^2 \right)  \\ &+ \frac{5 (v_2 + v_2^2)}{\sqrt{n \left( \frac{n}{2} - 1 \right)}},
\end{align*}
for $v_1 = B^2(\lambda) - V(\lambda) + \sigma^2 - 2\mathbb{E}_{\varepsilon}R_2$ and $\sigma^2 v_2 = V(\lambda) - B^2(\lambda) + 2\mathbb{E}_{\varepsilon}R_2 - \sigma^2$ .

In what follows, two cases are distinguished.

\textbf{Case 1:} \quad If $\lambda_1^{\tau} > \lambda_2^*$, then, by definition of $\lambda_2^*$, Corollary \ref{corollary_for_variance} and monotonicity of the variance term, 
\begin{equation} \label{case1}
    B^2(\lambda_1^{\tau}) < V(\lambda_1^{\tau}) + 2\mathbb{E}_{\varepsilon}R_2 - \sigma^2 \leq 2 V(\lambda_2^*) + 2\mathbb{E}_{\varepsilon}R_2 - \sigma^2 + (t - \Delta t)
\end{equation}
with probability at least $1 - 4 \exp \left( - c n \min \left( \frac{t - \Delta t}{\sigma^2}, \frac{(t - \Delta t)^2}{\sigma^4}, \frac{(t - \Delta t)^2}{\sigma^2} \right) \right), \ \forall t > 0$.

\textbf{Case 2:} \quad If $\lambda_1^{\tau} \leq \lambda_2^*$, then take $t - \Delta t$ from Ineq. (\ref{case1}) and define $\lambda^{\star \star}$ as in Eq. (\ref{k_two_star}).


Notice that $\lambda^{\star \star} \in \mathcal{K}_{B}$ by its definition. Therefore, due to Lemma \ref{left_deviation_ineq}, under the event $\mathcal{E}^{c}(\lambda^{\star \star})$, $\lambda_1^{\tau} > \lambda^{\star \star}$, and since $B^2(\lambda) \leq \lVert I_n - A_{\lambda} \rVert_2^2 \lVert f^* \rVert_n^2 \leq 4 \mathcal{M}^2$ for any $\lambda \in \{1, \ldots, n \}$,
\begin{equation*}
    B^2(\lambda_1^{\tau}) < V(\lambda_1^{\tau}) + 2\mathbb{E}_{\varepsilon}R_2 - \sigma^2 + t_1 \leq 2 V(\lambda_2^*) + 2\mathbb{E}_{\varepsilon}R_2 - \sigma^2  + (t - \Delta t) + t_1
\end{equation*}
with probability at least $1 - 4 \exp \left(- c n \min \left( \frac{(t - \Delta t)^2}{\sigma^2}, \frac{(t - \Delta t)^2}{\sigma^4}, \frac{t - \Delta t}{\sigma^2} \right) \right) \\ - 8 \exp \left( - c n \min \left( \frac{t_1^2}{\sigma^4}, \frac{t_1^2}{\sigma^2}, \frac{t_1}{\sigma^2} \right) \right)  - 5 \exp \left( - \min \left( 1, \sigma^2 / 128 \mathcal{M}^2 \right)  n t_2^2 \left( 1 - \frac{1}{n^{1.5}2^{n/2}} \right)^2 \right) - \frac{5 (t_2 + t_2^2)}{ \sqrt{n \left( \frac{n}{2} - 1 \right)} }$.

Combining \textbf{Case 1} and \textbf{Case 2} together,
\begin{equation}
    B^2(\lambda_1^{\tau}) < 2 V(\lambda_2^*) + 2\mathbb{E}_{\varepsilon}R_2 - \sigma^2 + (t - \Delta t) + t_1
\end{equation}
with probability at least $1 - 4 \exp \left(- c n \min \left( \frac{(t - \Delta t)^2}{\sigma^2}, \frac{(t - \Delta t)^2}{\sigma^4}, \frac{t - \Delta t}{\sigma^2} \right) \right) \\ - 8 \exp \left( - c n \min \left( \frac{t_1^2}{\sigma^4}, \frac{t_1^2}{\sigma^2}, \frac{t_1}{\sigma^2} \right) \right) - 5 \exp \left( - \min \left(1, \sigma^2 / 128 \mathcal{M}^2 \right) n t_2^2 \left( 1 - \frac{1}{n^{1.5}2^{n/2}} \right)^2 \right) - \frac{5 (t_2 + t_2^2)}{ \sqrt{n \left( \frac{n}{2} - 1 \right)} }$. The claim is proved.

\end{proof}

\section{Proof of Theorem \ref{main_th}} \label{proof_of_main_th}

Define $v(\lambda) \coloneqq \lVert A_{\lambda[k]} \varepsilon \rVert_n^2$, where $\lambda[k] = \textnormal{tr}(A_k) = n/k$ (see Section \ref{main_quantities} for the definitions related to the notation $\lambda$). Then, due to the inequality $(a + b)^2 \leq 2 a^2 + 2 b^2$ for any $a, b \geq 0$, Lemma \ref{deviation_bias_lemma}, Corollary \ref{corollary_for_variance}, and the control of the stochastic term in Appendix \ref{var_control}, for $\lambda_1
^{\tau}[k]$ and $\lambda_2^*[k]$ from Section \ref{main_quantities}, one obtains
\begin{align*}
    \lVert f^{\lambda_1^{\tau}[k]} - f^* \rVert_n^2 &= \lVert (I_n - A_{\lambda_1^{\tau}[k]})F^* \rVert_n^2 + \lVert A_{\lambda_1^{\tau}[k]}\varepsilon \rVert_n^2 + 2 \langle A_{\lambda_1^{\tau}[k]}\varepsilon, (I_n - A_{\lambda_1^{\tau}[k]})F^* \rangle_n \\
    &\leq 2 B^2(\lambda_1^{\tau}[k]) + 2 v(\lambda_1^{\tau}[k]) \\
    &\leq 4 V(\lambda_2^*[k]) + 2(2\mathbb{E}_{\varepsilon}R_2 - \sigma^2) + 2 V(\lambda_1^{\tau}[k]) + \widetilde{c}_1 \sqrt{\frac{\log n}{n}} + 6 t\\
    &\leq 8V(\lambda_2^*[k]) + 8t + 2(2\mathbb{E}_{\varepsilon} R_2 - \sigma^2) + \widetilde{c}_1 \sqrt{\frac{\log n}{n}} 
\end{align*}
with probability at least $1 - 18 \exp \left( -c_1 n \min \left( \frac{t^2}{\sigma^2}, \frac{t^2}{\sigma^4}, \frac{t}{\sigma^2} \right) \right) - 5 \exp \left( - \min \left( 1, \sigma^2 / 128 \mathcal{M}^2 \right)  n t_2^2 (1 - \frac{1}{n^{1.5}2^{n/2}})^2 \right) - \frac{5 (t_2 + t_2^2)}{ \sqrt{n \left( \frac{n}{2} - 1 \right)}}$, where $t \geq 0$ is arbitrary and $t_2 \geq \frac{4 \mathcal{M}^2}{\sigma^2}$.

In addition to that, if $\lambda_2^*$ from Eq. (\ref{stopping_times_tilde}) exists, then $V(\lambda_2^*[k]) \leq 1/2 \left[ \textnormal{MSE}(\lambda_2^*[k]) + \sigma^2 - 2\mathbb{E}_{\varepsilon}R_2 \right]$, and
\begin{equation}\label{ineq_for_corollary1}
    \lVert f^{\lambda_1^{\tau}[k]} - f^* \rVert_n^2 \leq 4 \textnormal{MSE}(\lambda_2^*[k]) - 2(2\mathbb{E}_{\varepsilon}R_2 - \sigma^2) + 8t + \widetilde{c}_1 \sqrt{\frac{\log n}{n}}
\end{equation}
with the same probability.

Define $u_1 \coloneqq c_1 n \min \left( \frac{t^2}{\sigma^2}, \frac{t^2}{\sigma^4}, \frac{t}{\sigma^2} \right)$, then one concludes that
\begin{equation} \label{ref_mse_high_prob_bound}
    \lVert f^{\lambda_1^{\tau}[k]} - f^* \rVert_n^2 \leq 4\textnormal{MSE}(\lambda_2^*[k]) + C \left( \sqrt{\frac{u_1 \sigma^2 }{n}} + \frac{u_1 \sigma^2}{n} + \sqrt{\frac{u_1 \sigma^4}{n}} \right) + \widetilde{c}_1 \sqrt{\frac{\log n}{n}} + 2 (\sigma^2 - 2\mathbb{E}_{\varepsilon}R_2)
\end{equation}
with probability at least $1 - 18 \exp (-u_1) - 5 \exp \left(- \min \left( 1, \frac{\sigma^2}{128 \mathcal{M}^2} \right) n t_2^2 \left(1 - \frac{1}{n^{1.5}2^{n/2}}\right)^2\right) - \frac{5 (t_2 + t_2^2)}{ \sqrt{n \left( \frac{n}{2} - 1 \right)}}$, where $u_1 \geq 0$ and $t_2 \geq \frac{4 \mathcal{M}^2}{\sigma^2}$, constant $\widetilde{c}_1$ may depend on $\sigma^2$, constant $C$ may depend on $\mathcal{M}, \sigma^2$.

\section{Proof of Corollary \ref{main_corollary}} \label{section_proof_corollary}

\subsection{Bounding the squared bias in the empirical norm and some definitions}

We start the proof with the argument that will help us transfer the result of Theorem \ref{main_th} in the empirical $\lVert \cdot \rVert_n^2$ norm to the $L_2(\mathbb{P}_X)$ norm $\lVert \cdot \rVert_2^2$. In order to do that, since there was no restriction on the fixed design of covariates $\{ x_1, \ldots, x_n \}$, we assume that the covariates lie in the regularly spaced lattice in $[0, 1]^d$. Then for any $f^* \in \mathcal{F}_{\textnormal{Lip}}(L)$, 
\begin{align*}
    B^2(k) = \lVert (I_n - A_k)F^* \rVert_n^2 &= \frac{1}{n}\sum_{i=1}^n \left[ f^*(x_i) - \frac{1}{k}\sum_{j \in \mathcal{N}_k(x_i)}f^*(x_j) \right]^2 \\
    &\leq \frac{1}{n}\sum_{i=1}^n \frac{1}{k^2} \left[ \sum_{j \in \mathcal{N}_k (x_i)} L \lVert x_i - x_j \rVert \right]^2\\
    &\leq (\sqrt{d}L)^2\left( \frac{k}{n} \right)^{2/d},
\end{align*}
Further, define 
\begin{equation} \label{bias_variance_df}
  k^{\textnormal{b/v}} \coloneqq \ceil*{\left( \frac{\sigma^2}{d L^2} \right)^{d/(2+d)} n^{\frac{2}{2 + d}}}  
\end{equation}
where the variance crosses the upper bound on squared bias.
%
Define 
\begin{equation} \label{delta_k_epsilon}
    \delta_{k, \epsilon} \coloneqq \left( 1 - \epsilon \right) \left[ \underset{f^* \in \mathcal{F}_{\textnormal{Lip}}(L)}{\sup} \left[ V(k) \right] + \underset{f^* \in \mathcal{F}_{\textnormal{Lip}}(L)}{\sup}[2B^2(2)] \right], \quad 0 < \epsilon < 1,
\end{equation}
for each $k = 1, \ldots, n$. After that, consider Lemma \ref{lemma_for_variance} and  define 
\begin{equation} \label{k_tilde_epsilon}
    \widetilde{k}_{\epsilon} \coloneqq \inf \left\{ k \in \{1, \ldots, n \} \mid B^2(k) \geq \epsilon \left[ \underset{f^* \in \mathcal{F}_{\textnormal{Lip}}(L)}{\sup} \left[ V(k) \right] + \underset{f^* \in \mathcal{F}_{\textnormal{Lip}}(L)}{\sup}[2B^2(2)] \right] \right\}.
\end{equation}
Define 
\begin{equation}
    \widetilde{k} \coloneqq \{ k \in \{1, \ldots, n \} \mid B^2(k) \geq V(k) \}.
\end{equation}
Notice that if one takes $\epsilon = \frac{V(k)}{\underset{f^* \in \mathcal{F}_{\textnormal{Lip}}(L)}{\sup}\left[ V(k) \right] + \underset{f^* \in \mathcal{F}_{\textnormal{Lip}}(L)}{\sup}[2B^2(2)]}$, then $\widetilde{k}_{\varepsilon}=\widetilde{k}$; if one takes $\epsilon \to 0, \ \widetilde{k}_{\epsilon} \to 1$. Moreover, $k^{\textnormal{b/v}}_{\epsilon} \leq \widetilde{k}_{\epsilon}$, where
\begin{equation}
    k^{\textnormal{b/v}}_{\epsilon} \coloneqq \ceil*{\left( \frac{\epsilon \sigma^2}{dL^2} \right)^{\frac{d}{2+d}} n^{\frac{2}{2+d}}}.
\end{equation}
\subsection{Change-of-norm argument}

For any $x \in \mathcal{X} \subseteq \mathbb{R}^d$, we will consider $\mid f^k (x) - f^* (x) \mid \leq \mid f^k(x) \mid + \mid f^*(x) \mid$. Recall that 
\begin{equation*}
    f^k(x) = \frac{1}{k}\sum_{j \in \mathcal{N}_k (x)}y_j,
\end{equation*}
where $\mathcal{N}_k(x)$ is the set of $k$ nearest neighbors of $x$ among $\{ x_1, \ldots, x_n \}$. 
We have $y_j \mid x_j \overset{\textnormal{i.i.d.}}{\sim} \mathcal{N}(f^*(x_j), \sigma^2)$, by using Lemma \ref{lemma_for_variance}, one has $k^{\tau} \geq \widetilde{k}_{\epsilon} - 1$ with high probability. For any $\widetilde{t} > 0$, 
\begin{align*}
    \mathbb{P}_{\varepsilon} \left( \underset{\widetilde{k}_{\epsilon} - 1 \leq k \leq n}{\max} \mid \frac{1}{k}\sum_{j \in \mathcal{N}_k(x)}\left( y_j - f^*(x_j)\right) \mid \geq \widetilde{t} \right) &\leq \underbrace{\mathbb{P}_{\varepsilon}\left( \underset{\widetilde{k}_{\epsilon} \leq k \leq n}{\max} \mid \frac{1}{k}\sum_{j \in \mathcal{N}_k(x)}\left( y_j - f^*(x_j)\right) \mid \geq \widetilde{\frac{t}{2}} \right)}_{p_1} \\
    &+ \underbrace{\mathbb{P}_{\varepsilon}\left( \mid \frac{1}{\widetilde{k}_{\epsilon}-1}\sum_{j \in \mathcal{N}_{\widetilde{k}_{\epsilon} - 1}(x)} \left( y_j - f^*(x_j) \right) \mid \geq \frac{\widetilde{t}}{2} \right)}_{p_2}.
\end{align*}
As for the first probability, since $n \geq \widetilde{k}_{\epsilon} \geq 2$,
\begin{align*}
    p_1 = \mathbb{P}_{\varepsilon} \left( \underset{\widetilde{k}_{\epsilon} \leq k \leq n}{\max} \mid \frac{1}{k}\sum_{j \in \mathcal{N}_k(x)}\left( y_j - f^*(x_j)\right) \mid \geq \frac{\widetilde{t}}{2} \right) &\leq 2 \cdot \sum_{k = \widetilde{k}_{\epsilon}}^n  \exp \left( - \frac{k \widetilde{t}^2}{8 \sigma^2} \right)\\
    &= \frac{2 e^{-\frac{\widetilde{t}^2}{8\sigma^2}}}{1 - e^{- \frac{\widetilde{t}^2}{8\sigma^2}}} \exp \left( - \frac{(\widetilde{k}_{\epsilon} - 1) \widetilde{t}^2}{8 \sigma^2} \right) \left[ 1  - e^{\frac{-\widetilde{t}^2}{8\sigma^2}(n + 1 - \widetilde{k}_{\epsilon})} \right]\\
    &\leq \frac{2 e^{-\frac{\widetilde{t}^2}{8\sigma^2}}}{1 - e^{-\frac{\widetilde{t}^2}{8\sigma^2}}}\exp \left(-\frac{(\widetilde{k}_\epsilon - 1)}{8\sigma^2} \right).
\end{align*}
Besides that,
\begin{equation*}
    p_2 \leq 2 \exp \left( - \frac{(\widetilde{k}_{\epsilon}-1)\widetilde{t}^2}{8\sigma^2}\right).
\end{equation*}
It implies that for $\widetilde{t} = 4\sigma \left( \frac{dL^2}{\sigma^2} \right)^{\frac{d}{2(2+d)}}$ and by using Lemma \ref{variance_aux_lemma},
\begin{equation*}
    \frac{1}{k^{\tau}}\sum_{j \in \mathcal{N}_{k^\tau}(x)}y_j \leq \frac{1}{k^\tau}\sum_{j \in \mathcal{N}_{k^\tau}(x)}f^*(x_j) + 4\sigma \left( \frac{dL^2}{\sigma^2} \right)^{\frac{d}{2(2+d)}}
\end{equation*}
with probability at least $1 - \left[1 + \left(1 - e^{-2\left(\frac{dL^2}{\sigma^2} \right)^{\frac{d}{2+d}}}\right)^{-1}\right]e^{-\frac{ \sigma^2\left( \frac{dL^2}{\sigma^2} \right)^{\frac{d}{2 + d}}}{V(\widetilde{k}_{\epsilon})}} \\- 4\exp\left(-c n \min \left( \frac{\delta_{\widetilde{k}_{\epsilon} - 1, \epsilon}^2}{\sigma^4}, \frac{\delta_{\widetilde{k}_{\epsilon} - 1, \epsilon}^2}{\sigma^2}, \frac{\delta_{\widetilde{k}_{\epsilon} - 1, \epsilon}}{\sigma^2} \right) \right)$.
Let us bound these probabilities. First, 
\begin{equation*}
    \exp \left( - \frac{\sigma^2 \left( \frac{d L^2}{\sigma^2} \right)}{V\left( \widetilde{k}_{\epsilon} \right)} \right) \leq \exp \left( - \left( \frac{d L^2}{\sigma^2} \right)^{\frac{d}{d+2}} \ceil*{\left( \frac{\epsilon \sigma^2}{d L^2} \right)^{\frac{d}{d+2}} n^{\frac{2}{d+2}} } \right) \leq \exp \left( - \epsilon^{\frac{d}{d+2}} n^{\frac{2}{d+2}} \right).
\end{equation*}
Second, set $\epsilon = \min \left[ 0.7, \max \left(0.5, \frac{0.5 \underset{f^* \in \mathcal{F}_{\textnormal{Lip}}(L)}{\sup}\left[ \textnormal{MSE}(k - 1) \right]}{\underset{f^* \in \mathcal{F}_{\textnormal{Lip}}(L)}{\sup}[2B^2(2)]} \right) \right]$ and assume w.l.o.g. that $\underset{f^* \in \mathcal{F}_{\textnormal{Lip}}(L)}{\sup} \left[ \textnormal{MSE} \left( \widetilde{k}_{\epsilon} - 1 \right)\right] \geq \underset{f^* \in \mathcal{F}_{\textnormal{Lip}}(L)}{\sup}[2B^2(2)]$. If the latter does not hold, the proof will be similar by using $\delta_{\widetilde{k}_{\epsilon} - 1, \epsilon} \geq \left( 1-\epsilon \right)\underset{f^* \in \mathcal{F}_{\textnormal{Lip}}(L)}{\sup}[2B^2(2)] > 0.5\cdot \underset{f^* \in \mathcal{F}_{\textnormal{Lip}}(L)}{\sup}\left[ \textnormal{MSE}\left(\widetilde{k}_{\epsilon}-1\right)\right]$. From the definition of $\widetilde{k}_\epsilon$ in Eq. (\ref{k_tilde_epsilon}), notice that for any $f^* \in \mathcal{F}_{\textnormal{Lip}}(L)$, 
\begin{equation*}
    \left( 1 - \epsilon \right)V\left( \widetilde{k}_{\varepsilon} - 1 \right) > \frac{1 - \epsilon}{1 + \epsilon}\left[ \textnormal{MSE}(\widetilde{k}_{\epsilon} - 1) - 2\epsilon B^2(2) \right].
\end{equation*}
It implies that
\begin{align}
    \underset{f^* \in \mathcal{F}_{\textnormal{Lip}}(L)}{\sup}\left[ V(\widetilde{k}_{\epsilon} - 1) \right] &> \underset{f^* \in \mathcal{F}_{\textnormal{Lip}}(L)}{\sup}\left[ \frac{1-\epsilon}{1+\epsilon}\left( \textnormal{MSE}(\widetilde{k}_\epsilon - 1) - 2\epsilon B^2(2) \right) \right]\\ 
    &\geq \frac{9}{170}\cdot \underset{f^* \in \mathcal{F}_{\textnormal{Lip}}(L)}{\sup}\left[ \textnormal{MSE}(\widetilde{k}_\epsilon - 1) \right].
\end{align}
Therefore, $\delta_{\widetilde{k}_{\epsilon} - 1, \epsilon} \geq (1 - \epsilon)\underset{f^* \in \mathcal{F}_{\textnormal{Lip}}(L)}{\sup}\left[V(\widetilde{k}_{\epsilon} - 1) \right] > \frac{9}{170}\underset{f^* \in \mathcal{F}_{\textnormal{Lip}}(L)}{\sup}\left[ \textnormal{MSE}\left( \widetilde{k}_{\epsilon} - 1 \right) \right] \geq c L^{\frac{2d}{d+2}}\left( \frac{\sigma^2}{n} \right)^{\frac{2}{d+2}}$ due to the fact that $\underset{f^* \in \mathcal{F}_{\textnormal{Lip}}(L)}{\sup}\left[ \textnormal{MSE}(k) \right] \geq c_l L^{\frac{2d}{d+2}}\left( \frac{\sigma^2}{n} \right)^{\frac{2}{d+2}}$ for any $k = 1, \ldots, n$ (cf. Lemma \ref{minimax_bound_emp_norm}), thus
\begin{align*}
    c n \frac{\delta_{\widetilde{k}_{\epsilon}-1, \epsilon}}{\sigma^2} &> \widetilde{c} n^{\frac{d}{d+2}},\\
    c n \frac{\delta_{\widetilde{k}_{\epsilon} - 1, \epsilon}^2}{\sigma^4} &> \widetilde{c}n^{\frac{d-2}{d+2}},\\
    c n \frac{\delta_{\widetilde{k}_{\epsilon}-1, \epsilon}^2}{\sigma^2} &> \widetilde{c}n^{\frac{d-2}{d+2}},
\end{align*}
where positive constant $\widetilde{c}$ depends on $d, L, \sigma^2$. Therefore, for any $x \in \mathcal{X}$,
\begin{equation}
    \mid f^{k^\tau}(x) - f^* (x) \mid \leq 6 \max \left(L, \sigma \left( \frac{dL^2}{\sigma^2} \right)^{\frac{d}{2(2 + d)}} \right)
\end{equation}
with high probability, and we can apply Lemma \ref{hoeffding_concentration} that gives us
\begin{equation} \label{ineq_for_the_norms}
    \lVert f^{k^\tau} - f^* \rVert_2^2 \leq \lVert f^{k^\tau} - f^* \rVert_n^2 + \epsilon 
\end{equation}
with probability at least $1 - \exp \left( -  \frac{n\epsilon^2}{2592 \max^4 \left(L, \sigma \left( \frac{dL^2}{\sigma^2} \right)^{\frac{d}{2(2+d)}} \right)}\right) -  \left[ 1 + \left(1 - e^{-4\left( \frac{dL^2}{\sigma^2} \right)^{\frac{d}{2+d}}}\right)^{-1}\right] \exp\left(-c_1 n^{\frac{2}{d+2}}\right) - 4\exp \left(- \widetilde{c}  n^{\frac{d-2}{d+2}} \right)$.

\subsection{Bounding the risk error in the population norm}

Recall that $\widetilde{k} = \inf \left\{ k \in \{ 1, \dots, n \} \mid B^2(k) \geq V(k) \right\}$, then $V(k^*) \leq V(\widetilde{k}) \leq V(k^{\textnormal{b/v}})$. 

Let us denote $u_2 = \min \left(1, \frac{\sigma^2}{128 \mathcal{M}^2} \right) n t_2^2 \left(1 - \frac{1}{n^{1.5}2^{n/2}}\right)^2$ from Ineq. (\ref{ineq_for_corollary1}), then 
\begin{equation*}
    t_2^2 = \frac{ u_2}{\min \left(1, \sigma^2 / 128 \mathcal{M}^2 \right) n \left(1 - \frac{1}{n^{1.5}2^{n/2}}\right)^2}.
\end{equation*}
Combining it with $2 B^2(2)  \leq \frac{c}{n^{2/d}} \leq \frac{c}{n^{2/(2+d)}}$ where positive $c$ depends on $d$ and $L$, Ineq. (\ref{ineq_for_the_norms}) for $\epsilon = \max^2 \left(L, \sigma \left( \frac{dL^2}{\sigma^2} \right)^{\frac{d}{2(2+d)}} \right) \sqrt{2592\frac{\log n}{n}}$ and Ineq. (\ref{ref_mse_high_prob_bound}) such that $u_1 = \log n$ and $u_2 = n \left(1 - \frac{1}{n^{1.5}2^{n/2}}\right)^2 \frac{\mathcal{M}^2}{\sigma^2}$, gives $t_2^2 \geq \frac{16 \mathcal{M}^4}{\sigma^4}$, and
\tiny
\begin{align} \label{to_exp_from_proba}
    \begin{split} 
     \mathbb{E} \lVert f^{k^{\tau}} - f^* \rVert_2^2 &= \mathbb{E} \left[ \lVert f^{k^{\tau}} - f^* \rVert_2^2 \mathbb{I}\left\{ \lVert f^{k^{\tau}} - f^* \rVert_2^2 \leq 8 V(k^{\textnormal{b/v}}) + C_1 \left(\frac{\sqrt{\log(n) \sigma^4}}{\sqrt{n}} + \frac{\sqrt{\log(n) \sigma^2  }}{\sqrt{n}} + \frac{\log(n) \sigma^2}{n} \right) + \frac{c}{n^{2/(d+2)}} \right\} \right] \\
    &+ \mathbb{E} \left[ \lVert f^{k^{\tau}} - f^* \rVert_2^2 \mathbb{I}\left\{ \lVert f^{k^{\tau}} - f^* \rVert_2^2 > 8 V(k^{\textnormal{b/v}}) + C_1 \left(\frac{\sqrt{\log(n) \sigma^4}}{\sqrt{n}} + \frac{\sqrt{\log(n) \sigma^2 }}{\sqrt{n}} + \frac{\log(n) \sigma^2}{n} \right) + \frac{c}{n^{2/(d+2)}} \right\}  \right].
    \end{split}
\end{align}
\small
After that, due to Lemma \ref{operator_norm_lemma} from the supplementary material, $\lVert I_n - A_k \rVert_2 \leq 2$ for any $k \in \{1, \ldots, n\}$,  and $| f^*(x_i) | \leq L$ for $i \in \{1, \ldots, n\}$ due to the Lipschitz condition (\ref{lipschitz}), which implies that,
\begin{align*}
    \lVert f^{k^{\tau}} - f^* \rVert_2^2 &= \mathbb{E}_{X}\left( \frac{1}{k^{\tau}}\sum_{j \in \mathcal{N}_{k^{\tau}}(X)}y_j - f^*(X) \right)^2 \\
    &= \int_{\mathcal{X}}\left( \frac{1}{k^{\tau}}\sum_{j \in \mathcal{N}_{k^{\tau}}(x)}y_j - f^*(x) \right)^2 d \mathbb{P}_X(x)\\
    &\leq \int_{\mathcal{X}}\frac{1}{(k^{\tau})^2}\sum_{j \in \mathcal{N}_{k^{\tau}}(x)}\left[ 2 \mid f^*(x_j) - f^*(x) \mid^2 + 2 \varepsilon_j^2 \right] d \mathbb{P}_X(x)\\
    &\leq \int_{\mathcal{X}}\sum_{j \in \mathcal{N}_{k^{\tau}}(x)}\frac{1}{(k^{\tau})^2}\left[ 8L^2 + 2 \varepsilon_j^2 \right]d \mathbb{P}_X (x)\\
    &= I_1 + I_2,
\end{align*}
where 
\begin{align*}
    I_1 &\leq 8L^2,\\
    I_2 &= \int_{\mathcal{X}}\sum_{j \in \mathcal{N}_{k^{\tau}}(x)}\frac{2\varepsilon_j^2}{(k^{\tau})^2}d \mathbb{P}_X (x).
\end{align*}
First, consider the random variable $S = \sum_{j \in \mathcal{N}_{k^{\tau}}(x)}\frac{\varepsilon_j^2}{(k^{\tau})^2}$, then
\begin{align*} 
    \mathbb{E}S &= \sum_{k=1}^n \mathbb{E}\left[ S \mid k^{\tau} = k \right] \mathbb{P}\left[ k^{\tau} = k \right] \\
    &= \sum_{k=1}^n \frac{\sigma^2}{k}\mathbb{P}\left[ k^{\tau} = k \right]\\
    &\leq \sigma^2.
\end{align*}
Thus, $\mathbb{E}\lVert f^{k^{\tau}} - f^* \rVert_2^2 \leq 2 \max \left( \sigma^2, 8L^2 \right)$. Applying the Cauchy-Schwarz inequality, gives 
\begin{equation} \label{aux_inequal}
    \mathbb{E}\lVert f^{k^{\tau}} - f^* \rVert_2^4 \leq \sqrt{\mathbb{E}\lVert f^{k^{\tau}} - f^* \rVert_2^2 \cdot \mathbb{E}\lVert f^{k^{\tau}} - f^* \rVert_2^2} \leq 2\max \left( \sigma^2, 8L^2 \right).
\end{equation}
From Ineq. (\ref{to_exp_from_proba}) and the Cauchy-Schwarz inequality again, it comes
%
\tiny
\begin{align*}
    &\mathbb{E} \lVert f^{k^{\tau}} - f^* \rVert_2^2 \leq  \frac{8\sigma^2}{k^{\textnormal{b/v}}} + C_1 \left( \frac{\sqrt{\log n \sigma^4}}{\sqrt{n}} + \frac{ \sqrt{\log n \sigma^2}}{\sqrt{n}} + \frac{\log n \sigma^2}{n} \right) + \frac{c}{n^{2/(2+d)}} \\
    &+ \sqrt{\mathbb{E}\lVert f^{k^{\tau}} - f^* \rVert_2^4}\sqrt{\mathbb{P}\left(\lVert f^{k^{\tau}} - f^* \rVert_2^2 > 8 V(k^{\textnormal{b/v}}) + C_1 \left( \frac{\sqrt{\log n \sigma^4}}{\sqrt{n}} + \frac{\sqrt{\log n \sigma^2  }}{\sqrt{n}} + \frac{\log n \sigma^2}{n} \right) + \frac{c}{n^{2/(2+d)}} \right)}. 
\end{align*}
\small
Combining Ineq. (\ref{to_exp_from_proba}), Ineq. (\ref{ineq_for_the_norms}), and Ineq. (\ref{aux_inequal}), we obtain
\begin{align*}
    &\mathbb{E}\lVert f^{k^{\tau}} - f^* \rVert_2^2 \leq  \frac{8\sigma^2}{k^{\textnormal{b/v}}}  + C_1 
  \left( \frac{\sqrt{\log n \sigma^4}}{\sqrt{n}} + \frac{\sqrt{\log n \sigma^2}}{\sqrt{n}} + \frac{\log n \sigma^2}{n} \right) + \frac{c}{n^{2/(d+2)}} \\&+ \sqrt{2\max \left( \sigma^2, 8L^2 \right)} \sqrt{\frac{19}{n} + 5 \exp\left(-\frac{L^2}{\sigma^2} n\left(1 - \frac{1}{n^{1.5}2^{n/2}}\right)^2\right) + \frac{5 c_1}{\sqrt{n \left( n/2 - 1 \right)}} + C_1 \exp \left( -c_1 n^{\frac{2}{d+2}} \right) + 4\exp \left( - \widetilde{c} n^{\frac{d-2}{d+2}} \right)}.
\end{align*}
The claim follows from the facts that $\frac{1}{\sqrt{n \left( n/2 - 1 \right)}} \leq \frac{5}{n}$ for $n \geq 3$, $\exp \left( -n \left( 1 - \frac{1}{n^{1.5}2^{n/2}} \right)^2 \right) \leq \frac{1}{n}$, and the definition (\ref{bias_variance_df}) of $k^{\textnormal{b/v}}$.

\end{document}